\documentclass[10pt]{article} %
\usepackage[preprint]{tmlr}

\usepackage{tcolorbox}
\usepackage{xspace}
\usepackage{algorithm}
\usepackage{algpseudocode}
\usepackage{bm}
\usepackage{stackrel}
\usepackage{epsf}
\usepackage{amsmath,amssymb}
\usepackage{graphicx}
\usepackage{color}
\usepackage{multirow,tabularx}
\usepackage{ifthen}
\usepackage{epstopdf}
\usepackage{caption}
\usepackage{subcaption}
\usepackage{wrapfig}

\usepackage{bbding}
\usepackage{amsthm}
\usepackage{thmtools}
\usepackage{xpatch}

\usepackage{hyperref}

\usepackage{lipsum}

\title{Outlier-Robust Training of Machine Learning Models}

\author{\name Rajat Talak \email talak@mit.edu \\
       \addr Laboratory of Information \& Decision Systems\\
       Massachusetts Institute of Technology\\
       Cambridge, MA 02139, USA       
       \AND
       \name Charis Georgiou \email cgeo@mit.edu \\
       \addr Department of Electrical Engineering and Computer Science\\
       Massachusetts Institute of Technology\\
       Cambridge, MA 02139, USA  
       \AND 
       \name Jingnan Shi \email jnshi@mit.edu \\
       \addr Laboratory of Information \& Decision Systems\\
       Massachusetts Institute of Technology\\
       Cambridge, MA 02139, USA  
       \AND
       \name Luca Carlone \email lcarlone@mit.edu \\
       \addr Laboratory of Information \& Decision Systems\\
       Massachusetts Institute of Technology\\
       Cambridge, MA 02139, USA  
       }

\begin{document}

\newtheorem{theorem}{Theorem}
\newtheorem{problem}{Problem}
\newtheorem{trule}[theorem]{Rule}
\newtheorem{corollary}[theorem]{Corollary}
\newtheorem{procedure}{\textbf{Procedure}}
\newtheorem{conjecture}[theorem]{Conjecture}
\newtheorem{lemma}[theorem]{Lemma}
\newtheorem{assumption}[theorem]{Assumption}
\newtheorem{definition}[theorem]{Definition}
\newtheorem{proposition}[theorem]{Proposition}
\newtheorem{remark}[theorem]{Remark}
\newtheorem{example}[theorem]{Example}

\newcommand{\cf}{\emph{cf.}\xspace}
\newcommand{\Figure}{Fig.}
\newcommand{\bdmath}{\begin{dmath}}
\newcommand{\edmath}{\end{dmath}}
\newcommand{\beq}{\begin{equation}}
\newcommand{\eeq}{\end{equation}}
\newcommand{\bdm}{\begin{displaymath}}
\newcommand{\edm}{\end{displaymath}}
\newcommand{\bea}{\begin{eqnarray}}
\newcommand{\eea}{\end{eqnarray}}
\newcommand{\beal}{\beq \begin{array}{ll}}
\newcommand{\eeal}{\end{array} \eeq}
\newcommand{\beas}{\begin{eqnarray*}}
\newcommand{\eeas}{\end{eqnarray*}}
\newcommand{\ba}{\begin{array}}
\newcommand{\ea}{\end{array}}
\newcommand{\bit}{\begin{itemize}}
\newcommand{\eit}{\end{itemize}}
\newcommand{\ben}{\begin{enumerate}}
\newcommand{\een}{\end{enumerate}}
\newcommand{\expl}[1]{&&\qquad\text{\color{gray}(#1)}\nonumber}

\makeatletter
\let\proof\thmt@original@proof
\xpatchcmd{\proof}{\itshape}{\bfseries}{}{}
\let\thmt@original@proof\proof
\makeatother

\newcommand{\newDay}[1]{
	\noindent\rule{\linewidth}{1pt}
	\texttt{Date: #1} \\
	\rule{\linewidth}{1pt}
}

\newcommand{\calA}{{\cal A}}
\newcommand{\calB}{{\cal B}}
\newcommand{\calC}{{\cal C}}
\newcommand{\calD}{{\cal D}}
\newcommand{\calE}{{\cal E}}
\newcommand{\calF}{{\cal F}}
\newcommand{\calG}{{\cal G}}
\newcommand{\calH}{{\cal H}}
\newcommand{\calI}{{\cal I}}
\newcommand{\calJ}{{\cal J}}
\newcommand{\calK}{{\cal K}}
\newcommand{\calL}{{\cal L}}
\newcommand{\calM}{{\cal M}}
\newcommand{\calN}{{\cal N}}
\newcommand{\calO}{{\cal O}}
\newcommand{\calP}{{\cal P}}
\newcommand{\calQ}{{\cal Q}}
\newcommand{\calR}{{\cal R}}
\newcommand{\calS}{{\cal S}}
\newcommand{\calT}{{\cal T}}
\newcommand{\calU}{{\cal U}}
\newcommand{\calV}{{\cal V}}
\newcommand{\calW}{{\cal W}}
\newcommand{\calX}{{\cal X}}
\newcommand{\calY}{{\cal Y}}
\newcommand{\calZ}{{\cal Z}}

\newcommand{\setA}{\textsf{A}}
\newcommand{\setB}{\textsf{B}}
\newcommand{\setC}{\textsf{C}}
\newcommand{\setD}{\textsf{D}}
\newcommand{\setE}{\textsf{E}}
\newcommand{\setF}{\textsf{F}}
\newcommand{\setG}{\textsf{G}}
\newcommand{\setH}{\textsf{H}}
\newcommand{\setI}{\textsf{I}}
\newcommand{\setJ}{\textsf{J}}
\newcommand{\setK}{\textsf{K}}
\newcommand{\setL}{\textsf{L}}
\newcommand{\setM}{\textsf{M}}
\newcommand{\setN}{\textsf{N}}
\newcommand{\setO}{\textsf{O}}
\newcommand{\setP}{\textsf{P}}
\newcommand{\setQ}{\textsf{Q}}
\newcommand{\setR}{\textsf{R}}
\newcommand{\setS}{\textsf{S}}
\newcommand{\setT}{\textsf{T}}
\newcommand{\setU}{\textsf{U}}
\newcommand{\setV}{\textsf{V}}
\newcommand{\setW}{\textsf{W}}
\newcommand{\setX}{\textsf{X}}
\newcommand{\setY}{\textsf{Y}}
\newcommand{\setZ}{\textsf{Z}}

\newcommand{\smallheading}[1]{\textit{#1}: }
\newcommand{\algostep}[1]{{\small\texttt{#1:}}\xspace}
\newcommand{\etal}{\emph{et~al.}\xspace}
\newcommand{\setal}{~\emph{et~al.}\xspace}
\newcommand{\eg}{\emph{e.g.,}\xspace}
\newcommand{\ie}{\emph{i.e.,}\xspace}
\newcommand{\myParagraph}[1]{\textbf{#1.}\xspace}
\newcommand{\M}[1]{{\bm #1}} %
\renewcommand{\boldsymbol}[1]{{\bm #1}}
\newcommand{\algoname}[1]{\textsc{#1}} %

\newcommand{\towrite}{{\color{red} To write}\xspace}
\newcommand{\xxx}{{\color{red} XXX}\xspace}
\newcommand{\AC}[1]{{\color{red} \textbf{AC}: #1}}
\newcommand{\LC}[1]{{\color{red} \textbf{LC}: #1}}
\newcommand{\FM}[1]{{\color{red} \textbf{FM}: #1}}
\newcommand{\RT}[1]{{\color{blue} \textbf{RT}: #1}}
\newcommand{\VT}[1]{{\color{blue} \textbf{VT}: #1}}
\newcommand{\hide}[1]{}
\newcommand{\wrt}{w.r.t.\xspace}
\newcommand{\highlight}[1]{{\color{red} #1}}
\newcommand{\tocheck}[1]{{\color{brown} #1}}
\newcommand{\grayout}[1]{{\color{gray} #1}}
\newcommand{\grayText}[1]{{\color{gray} \text{#1} }}
\newcommand{\grayMath}[1]{{\color{gray} #1 }}
\newcommand{\hiddenText}{{\color{gray} hidden text.}}
\newcommand{\hideWithText}[1]{\hiddenText}

\newcommand{\NA}{{\sf n/a}}
\newcommand{\versus}{\scenario{VS}\xspace}

\newcommand{\kron}{\otimes}
\newcommand{\dist}{\mathbf{dist}}
\newcommand{\iter}{\! \rm{iter.} \;}
\newcommand{\leqt}{\!\!\! < \!\!\!}
\newcommand{\geqt}{\!\!\! > \!\!\!}
\newcommand{\mysetminus}{-} %
\newcommand{\powerset}{\mathcal{P}}
\newcommand{\Int}[1]{ { {\mathbb Z}^{#1} } }
\newcommand{\Natural}[1]{ { {\mathbb N}^{#1} } }
\newcommand{\Complex}[1]{ { {\mathbb C}^{#1} } }
\newcommand{\one}{ {\mathbf{1}} }
\newcommand{\subject}{\text{ subject to }}
\newcommand{\normsq}[2]{\left\|#1\right\|^2_{#2}}
\newcommand{\norm}[1]{\left\| #1 \right\|}
\newcommand{\normsqs}[2]{\|#1\|^2_{#2}}
\newcommand{\infnorm}[1]{\left\|#1\right\|_{\infty}}
\newcommand{\zeronorm}[1]{\|#1\|_{0}}
\newcommand{\onenorm}[1]{\|#1\|_{1}}
\newcommand{\lzero}{\ell_{0}}
\newcommand{\lone}{\ell_{1}}
\newcommand{\linf}{\ell_{\infty}}

\newcommand{\E}{{\mathbb{E}}}
\newcommand{\EV}{\mathbb{E}}
\newcommand{\erf}{{\mathbf{erf}}}
\newcommand{\prob}[1]{{\mathbb P}\left(#1\right)}
\newcommand{\tran}{^{\mathsf{T}}}
\newcommand{\traninv}{^{-\mathsf{T}}}
\newcommand{\diag}[1]{\mathrm{diag}\left(#1\right)}
\newcommand{\trace}[1]{\mathrm{tr}\left(#1\right)}
\newcommand{\conv}[1]{\mathrm{conv}\left(#1\right)}
\newcommand{\polar}[1]{\mathrm{polar}\left(#1\right)}
\newcommand{\rank}[1]{\mathrm{rank}\left(#1\right)}
\newcommand{\e}{{\mathrm e}}
\newcommand{\inv}{^{-1}}
\newcommand{\pinv}{^\dag}
\newcommand{\until}[1]{\{1,\dots, #1\}}
\newcommand{\ones}{{\mathbf 1}}
\newcommand{\zero}{{\mathbf 0}}
\newcommand{\eye}{{\mathbf I}}
\newcommand{\vect}[1]{\left[\begin{array}{c}  #1  \end{array}\right]}
\newcommand{\matTwo}[1]{\left[\begin{array}{cc}  #1  \end{array}\right]}
\newcommand{\matThree}[1]{\left[\begin{array}{ccc}  #1  \end{array}\right]}
\newcommand{\dss}{\displaystyle}
\newcommand{\Real}[1]{ { {\mathbb R}^{#1} } }
\newcommand{\reals}{\Real{}}
\newcommand{\opt}{^{\star}}
\newcommand{\only}{^{\alpha}}
\newcommand{\copt}{^{\text{c}\star}}
\newcommand{\atk}{^{(k)}}
\newcommand{\att}{^{(t)}}
\newcommand{\at}[1]{^{(#1)}}
\newcommand{\attau}{^{(\tau)}}
\newcommand{\atc}[1]{^{(#1)}}
\newcommand{\atK}{^{(K)}}
\newcommand{\atj}{^{(j)}}
\newcommand{\projector}{{\tt projector}}
\newcommand{\setdef}[2]{ \{#1 \; {:} \; #2 \} }
\newcommand{\smalleye}{\left(\begin{smallmatrix}1&0\\0&1\end{smallmatrix}\right)}

\newcommand{\SEtwo}{\ensuremath{\mathrm{SE}(2)}\xspace}
\newcommand{\SE}[1]{\ensuremath{\mathrm{SE}(#1)}\xspace}
\newcommand{\SEthree}{\ensuremath{\mathrm{SE}(3)}\xspace}
\newcommand{\SOtwo}{\ensuremath{\mathrm{SO}(2)}\xspace}
\newcommand{\SOthree}{\ensuremath{\mathrm{SO}(3)}\xspace}
\newcommand{\Othree}{\ensuremath{\mathrm{O}(3)}\xspace}
\newcommand{\SOn}{\ensuremath{\mathrm{SO}(n)}\xspace}
\newcommand{\SO}[1]{\ensuremath{\mathrm{SO}(#1)}\xspace}
\newcommand{\On}{\ensuremath{\mathrm{O}(n)}\xspace}
\newcommand{\sotwo}{\ensuremath{\mathrm{so}(2)}\xspace}
\newcommand{\sothree}{\ensuremath{\mathrm{so}(3)}\xspace}
\newcommand{\intexpmap}[1]{\mathrm{Exp}\left(#1\right)}
\newcommand{\intlogmap}[1]{\mathrm{Log}\left(#1\right)}
\newcommand{\logmapz}[1]{\mathrm{Log}_0(#1)}
\newcommand{\loglikelihood}{\mathrm{log}\calL}
\newcommand{\intprinlogmap}{\mathrm{Log}_0}
\newcommand{\expmap}[1]{\intexpmap{#1}}
\newcommand{\expmaps}[1]{\langle #1 \rangle_{2\pi}}
\newcommand{\biggexpmap}[1]{\left\langle #1 \right\rangle_{2\pi}}
\newcommand{\logmap}[1]{\intlogmap{#1}}
\newcommand{\round}[1]{\mathrm{round}\left( #1 \right)}
\newcommand{\Rthree}{\ensuremath{\mathbb{R}^3}\xspace}
\newcommand{\Rtwo}{\ensuremath{\mathbb{R}^{2}}\xspace}
\newcommand{\R}[1]{\ensuremath{\mathbb{R}^{#1}}\xspace}

\newcommand{\MA}{\M{A}}
\newcommand{\MB}{\M{B}}
\newcommand{\MC}{\M{C}}
\newcommand{\MD}{\M{D}}
\newcommand{\ME}{\M{E}}
\newcommand{\MJ}{\M{J}}
\newcommand{\MK}{\M{K}}
\newcommand{\MG}{\M{G}}
\newcommand{\MM}{\M{M}}
\newcommand{\MN}{\M{N}}
\newcommand{\MP}{\M{P}}
\newcommand{\MQ}{\M{Q}}
\newcommand{\MU}{\M{U}}
\newcommand{\MR}{\M{R}}
\newcommand{\MS}{\M{S}}
\newcommand{\MI}{\M{I}}
\newcommand{\MV}{\M{V}}
\newcommand{\MF}{\M{F}}
\newcommand{\MH}{\M{H}}
\newcommand{\ML}{\M{L}}
\newcommand{\MO}{\M{O}}
\newcommand{\MT}{\M{T}}
\newcommand{\MX}{\M{X}}
\newcommand{\MY}{\M{Y}}
\newcommand{\MW}{\M{W}}
\newcommand{\MZ}{\M{Z}}
\newcommand{\MSigma}{\M{\Sigma}}
\newcommand{\MOmega}{\M{\Omega}}
\newcommand{\MPhi}{\M{\Phi}}
\newcommand{\MPsi}{\M{\Psi}}
\newcommand{\MDelta}{\M{\Delta}}
\newcommand{\MLambda}{\M{\Lambda}}

\newcommand{\vzero}{\boldsymbol{0}}
\newcommand{\vone}{\boldsymbol{1}}
\newcommand{\va}{\boldsymbol{a}}
\newcommand{\vh}{\boldsymbol{h}}
\newcommand{\vb}{\boldsymbol{b}}
\newcommand{\vc}{\boldsymbol{c}}
\newcommand{\vd}{\boldsymbol{d}}
\newcommand{\ve}{\boldsymbol{e}}
\newcommand{\vf}{\boldsymbol{f}}
\newcommand{\vg}{\boldsymbol{g}}
\newcommand{\vk}{\boldsymbol{k}}
\newcommand{\vl}{\boldsymbol{l}}
\newcommand{\vn}{\boldsymbol{n}}
\newcommand{\vo}{\boldsymbol{o}}
\newcommand{\vp}{\boldsymbol{p}}
\newcommand{\vq}{\boldsymbol{q}}
\newcommand{\vr}{\boldsymbol{r}}
\newcommand{\vs}{\boldsymbol{s}}
\newcommand{\vu}{\boldsymbol{u}}
\newcommand{\vv}{\boldsymbol{v}}
\newcommand{\vt}{\boldsymbol{t}}
\newcommand{\vxx}{\boldsymbol{x}}
\newcommand{\vy}{\boldsymbol{y}}
\newcommand{\vw}{\boldsymbol{w}}
\newcommand{\vzz}{\boldsymbol{z}}
\newcommand{\vdelta}{\boldsymbol{\delta}}
\newcommand{\vgamma}{\boldsymbol{\gamma}}
\newcommand{\vlambda}{\boldsymbol{\lambda}}
\newcommand{\vtheta}{\boldsymbol{\theta}}
\newcommand{\valpha}{\boldsymbol{\alpha}}
\newcommand{\vbeta}{\boldsymbol{\beta}}
\newcommand{\vnu}{\boldsymbol{\nu}}
\newcommand{\vmu}{\boldsymbol{\mu}}
\newcommand{\vepsilon}{\boldsymbol{\epsilon}}
\newcommand{\vtau}{\boldsymbol{\tau}}

\newcommand{\Rtheta}{\boldsymbol{R}}
\newcommand{\symf}{f} %

\newcommand{\angledomain}{(-\pi,+\pi]}

\newcommand{\MCB}{\mathsf{MCB}}
\newcommand{\FCM}{\mathsf{FCM}}
\newcommand{\FCB}{\mathsf{FCB}}
\newcommand{\cyclemap}[1]{\calC^{\calG}\left(#1\right)}
\newcommand{\incidencemap}[1]{\calA^{\calG}\left(#1\right)}
\newcommand{\cyclemapk}{\calC^{\calG}_{k}}
\newcommand{\incidencemapij}{\calA^{\calG}_{ij}}
\renewcommand{\ij}{_{ij}}
\newcommand{\foralledges}{\forall(i,j) \in \calE}
\newcommand{\sumalledges}{
     \displaystyle
     \sum_{(i,j) \in \calE}}
\newcommand{\sumalledgesm}{
     \displaystyle
     \sum_{i=1}^{m}}
\newcommand{\T}{\mathsf{T}}
\newcommand{\To}{\T_{\rm o}}
\newcommand{\Tm}{\T_{\rm m}}
\newcommand{\MCBa}{\MCB_{\mathsf{a}}}
\newcommand{\FCBo}{\FCB_{\mathsf{o}}}
\newcommand{\FCBm}{\FCB_{\mathsf{m}}}

\newcommand{\algoonlyname}{MOLE2D}
\newcommand{\algoml}{{\smaller\sf \algoonlyname}\xspace}
\newcommand{\algocyclebasis}{\algoname{compute-cycle-basis}}
\newcommand{\scenario}[1]{{\smaller \sf#1}\xspace}
\newcommand{\toro}{{\smaller\sf Toro}\xspace}
\newcommand{\gtwoo}{{\smaller\sf g2o}\xspace}
\newcommand{\gtwooST}{{\smaller\sf g2oST}\xspace}
\newcommand{\gtwood}{{\smaller\sf g2o{10}}\xspace}
\newcommand{\gtsam}{{\smaller\sf gtsam}\xspace}
\newcommand{\isam}{{\smaller\sf iSAM}\xspace}
\newcommand{\lago}{{\smaller\sf LAGO}\xspace}
\newcommand{\egtwoo}{{\smaller\sf \algoonlyname+g2o}\xspace}

\newcommand{\rim}{\scenario{rim}}
\newcommand{\cubicle}{\scenario{cubicle}}
\newcommand{\sphere}{\scenario{sphere}}
\newcommand{\sphereHard}{\scenario{sphere-a}}
\newcommand{\garage}{\scenario{garage}}
\newcommand{\torus}{\scenario{torus}}
\newcommand{\oneloop}{\scenario{circle}}
\newcommand{\intel}{\scenario{INTEL}}
\newcommand{\bovisa}{\scenario{Bovisa}}
\newcommand{\bov}{\scenario{B25b}}
\newcommand{\fra}{\scenario{FR079}}
\newcommand{\frb}{\scenario{FRH}}
\newcommand{\csail}{\scenario{CSAIL}}
\newcommand{\Ma}{\scenario{M3500}}
\newcommand{\Mb}{\scenario{M10000}}
\newcommand{\ATE}{\scenario{ATE}}
\newcommand{\CVX}{\scenario{CVX}}
\newcommand{\NEOS}{\scenario{NEOS}}
\newcommand{\sdptThree}{\scenario{sdpt3}}
\newcommand{\MOSEK}{\scenario{MOSEK}}
\newcommand{\NESTA}{\scenario{NESTA}}
\newcommand{\vertigo}{\scenario{Vertigo}}
\newcommand{\SDPA}{\scenario{SDPA}}

\newcommand{\cvx}{{\sf cvx}\xspace}

\newcommand{\blue}[1]{{\color{blue}#1}}
\newcommand{\green}[1]{{\color{green}#1}}
\newcommand{\red}[1]{{\color{red}#1}}

\newcommand{\linkToPdf}[1]{\href{#1}{\blue{(pdf)}}}
\newcommand{\linkToPpt}[1]{\href{#1}{\blue{(ppt)}}}
\newcommand{\linkToCode}[1]{\href{#1}{\blue{(code)}}}
\newcommand{\linkToWeb}[1]{\href{#1}{\blue{(web)}}}
\newcommand{\linkToVideo}[1]{\href{#1}{\blue{(video)}}}
\newcommand{\linkToMedia}[1]{\href{#1}{\blue{(media)}}}
\newcommand{\award}[1]{\xspace} %

\newcommand{\vpose}{\boldsymbol{x}}
\newcommand{\vz}{\boldsymbol{z}}
\newcommand{\vDelta}{\boldsymbol{\Delta}}
\newcommand{\vposesub}{\hat{\vpose}}
\newcommand{\vpossub}{\hat{\vpos}}
\newcommand{\vthetasub}{\hat{\vtheta}}
\newcommand{\Pthetasub}{\MP_\vtheta}
\newcommand{\thetasub}{\hat{\theta}}
\newcommand{\vposecorr}{\tilde{\vpose}}
\newcommand{\vposcorr}{\tilde{\vpos}}
\newcommand{\vthetacorr}{\tilde{\vtheta}}
\newcommand{\vposestar}{{\vpose}^{\star}}
\newcommand{\vposstar}{{\vpos}^{\star}}
\newcommand{\vthetastar}{{\vtheta}^{\star}}
\newcommand{\vcthetastar}{{\vctheta}^{\star}}
\newcommand{\pose}{\boldsymbol{x}}
\newcommand{\pos}{\boldsymbol{p}}
\newcommand{\mease}{z} %
\newcommand{\meas}{\boldsymbol{z}} %
\newcommand{\meashat}{\hat{\meas}}

\newcommand{\identity}[1]{\ensuremath{\mathbb{I}\{#1\}}\xspace}

\newtheorem{algo}{Algorithm}

\newcommand{\Nin}{\ensuremath{n_{I}}\xspace}
\newcommand{\Nout}{\ensuremath{n_{O}}\xspace}
\newcommand{\FracOut}{\ensuremath{\lambda}\xspace}

\newcommand{\cei}{\ensuremath{f^{\text{CE}}_i(\vw)}\xspace}

\newcommand{\Lp}{\ensuremath{l_p}\xspace}

\newcommand{\algoName}{\ensuremath{\text{AAA}}\xspace}
\newcommand{\algoNameLong}{adaptive alternation algorithm\xspace}
\newcommand{\algoNameLongCaps}{Adaptive Alternation Algorithm\xspace}
\newcommand{\algoNameT}{\ensuremath{\text{AAA}_T}\xspace}
\newcommand{\algoNameO}{\ensuremath{\text{AAA}_1}\xspace}
 
\maketitle

\begin{abstract}%
Robust training of machine learning models in the presence of outliers has garnered attention across various domains. 
The use of robust losses is a popular approach and is known to mitigate the impact of outliers. 
We bring to light two literatures that have diverged in their ways of designing robust losses: one using M-estimation, which is popular in robotics and computer vision, and another using a risk-minimization framework, which is popular in deep learning. 
We {first} show that a simple modification of the Black-Rangarajan duality provides a unifying view.
The modified duality brings out a definition of a \emph{robust loss kernel} $\sigma$ that is satisfied by robust losses in both the literatures.
{Secondly,} using the modified duality, we propose an
\emph{\algoNameLongCaps} (\algoName) for training machine learning models with outliers. 
The algorithm iteratively trains the model by using a weighted version of the non-robust loss, while updating the weights at each iteration.
The algorithm 
is 
augmented with a novel 
parameter update rule by interpreting the weights as inlier probabilities, and obviates the need for complex parameter tuning. 
{Thirdly},
we investigate convergence 
of the 
\algoNameLong
to outlier-free optima.
Considering arbitrary outliers (\ie with no distributional assumption on the outliers), we show that the use of robust loss kernels $\sigma$ increases the region of convergence. 
We experimentally show the efficacy of our algorithm on regression, classification, and neural scene reconstruction problems.%
\footnote{We release our implementation code: \href{https://github.com/MIT-SPARK/ORT}{https://github.com/MIT-SPARK/ORT}.}
\end{abstract}

\section{Introduction}
\label{sec:intro}
Humans are good at detecting and isolating outliers~\citep{Chai20icmd-HumanloopOutlier}. 
This is not the case when it comes to training machine learning models~\citep{Sukhbaatar15iclr-TrainingConvolutional, Wang24cvpr-BenchmarkingRobustness, Sabour23cvpr-robustnerf}.
Robustly training deep learning models in the presence of outliers is an important challenge. In particular, it can offset the high cost of obtaining accurate annotations. Many works now implement automatic or semi-automatic annotation pipelines which can be leveraged to train models \citep{Armeni16cvpr-3DsemanticParsing, %
Chang173dv-Matterport3D, Tkachenko20misc-LabelStudio, Yang21arxiv-Auto4DLearning, Gadre23nips-datacomp}. %
Recent efforts in robotics envision robots that
 can self-train their models by collecting and self-annotating data  \citep{Schmidt20rr-SelfdirectedLifelong, Deng20icra-Selfsupervised6D, Lu22iros-selftraining, Talak23tro-c3po, Shi23rss-ensemble, Jawaid24iros-cep, Wang24taes-domaingapSatellite}. %
\begin{wrapfigure}{r}{0.53\textwidth}
\centering
\vspace{-3mm}
\begin{subfigure}{.17\textwidth}
  \centering
  \includegraphics[width=\linewidth]{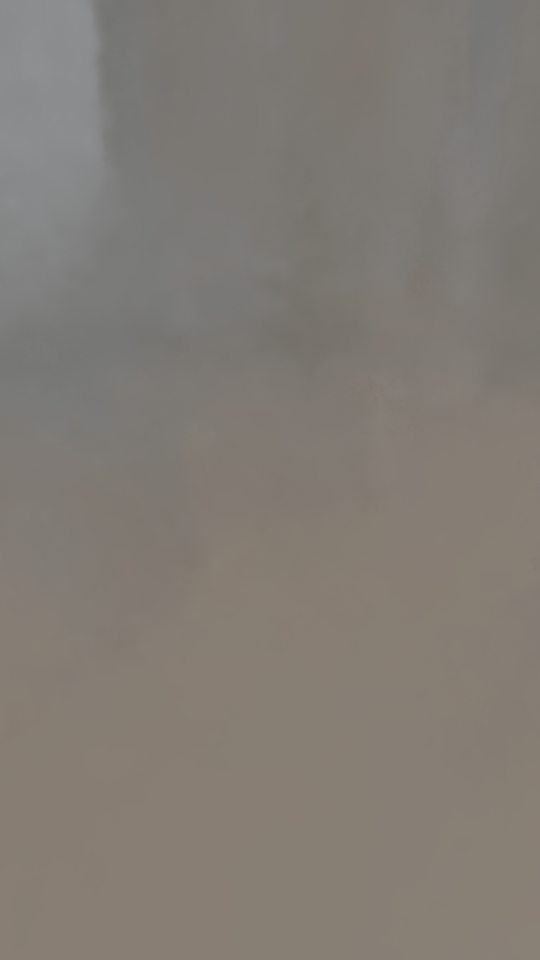}
  \label{fig:tl00}
\end{subfigure}
\begin{subfigure}{.17\textwidth}
  \centering
  \includegraphics[width=\linewidth]{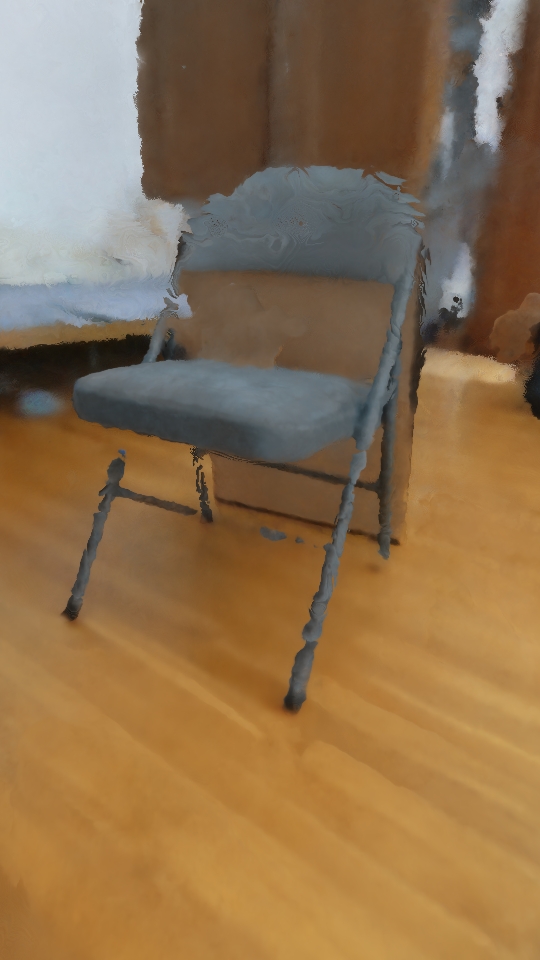}
  \label{fig:adam00}
\end{subfigure}
\begin{subfigure}{.17\textwidth}
  \centering
  \includegraphics[width=\linewidth]{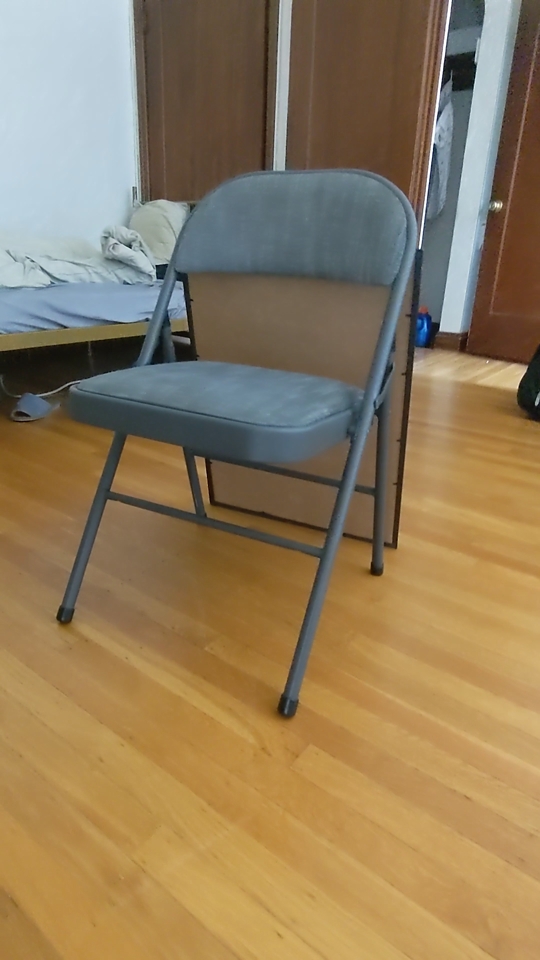}
  \label{fig:gt00}
\end{subfigure}
\vspace{-3mm}
\caption{
Nerfacto~\citep{Tancik23siggraph-nerfstudio} reconstruction results after $80\%$ of the training pixels have been perturbed by outliers. 
(left) Training with the original Adam optimizer. (middle) Training with our \algoNameLongCaps with Truncated Loss. (right) Ground truth. \label{fig:opening-nerfacto}
}
\vspace{-10mm}
\end{wrapfigure} 
In typical learning problems one computes the unknowns (\eg network weights) by optimizing a loss function $f_i$ for each training sample $i$:
\begin{equation}
\begin{aligned}
\label{eq:intro-original-obj}
& \underset{\vw \in \setW}{\text{Minimize}}
& &\sum_{i =1}^{n} f_i(\vw),
\end{aligned}
\end{equation}
where $\setW$ is the set of allowed parameters.
For instance, $f_i(\cdot)$ may be the cross-entropy loss or the $\ell_2$ norm squared measuring the mismatch between the $i$-th training label and the corresponding network prediction.

M-estimation~\citep{Huber81} suggests that in the presence of outliers, one needs to wrap typical losses into a robust loss function $\rho$:
\begin{equation}
\begin{aligned}
\label{eq:intro-m-est}
& \underset{\vw \in \setW}{\text{Minimize}}
& &\sum_{i =1}^{n} \rho(f_i(\vw)),
\end{aligned}
\end{equation}
where $\rho$ is responsible for mitigating the impact of terms with high loss (\ie high $f_i(\vw)$).
Many robust losses have been proposed in the literature to mitigate the effect of outliers. Recent works in robust estimation in robotics have shown that using a parameterized robust loss $\rho$, with adaptive parameter tuning during training, yields better outlier mitigation (see Section~\ref{sec:lit-robust-cv}). 
Many robust losses have also been proposed in training deep learning models for the task of multi-label classification (see Section~\ref{sec:lit-robust-dl}). However, we observe a divergence in the principles that govern the design of robust losses in (a) robotics and computer vision, where works mostly use robust estimation frameworks, and in (b) training deep learning models, which mostly relies on  risk-minimization frameworks (see Section~\ref{sec:back}). 

Robust estimation as applied in robotics and computer vision often relies on the insight that problem~\eqref{eq:intro-m-est} can be written down as a weighted least squares problem
\begin{equation}
\begin{aligned}
\label{eq:intro-m-est-dual}
& \underset{\vw \in \setW,~u_i \in [0, 1]}{\text{Minimize}}
& &\sum_{i =1}^{n} u_i \cdot f^2_i(\vw) + \Psi_{\rho}(u_i),
\end{aligned}
\end{equation}
where $\Psi_\rho(u)$ is an outlier process that is determined by the Black-Rangarajan duality~\citep{Black96ijcv-unification}. The equivalence between~\eqref{eq:intro-m-est} and~\eqref{eq:intro-m-est-dual} is useful for robotics and computer vision applications as common robust estimation problems can be re-written as weighted non-linear least squares, which are typically easier to solve. 
However, this framework cannot be applied directly to machine learning problems. For example, if $f_i(\vw)$ is the cross-entropy loss, minimizing the squared cross-entropy loss does not make an equal sense. 

On the other hand, when we consider classification problems in machine learning, 
the literature uses a risk-minimization framework to develop the notion of noise-tolerant loss~\citep{Ghosh15arxiv-MakingRisk, Ghosh17arxiv-RobustLoss}.  
Let model weight $\vw_\FracOut$ minimize risk when there are $\FracOut$ fraction of outliers. \cite{Ghosh15arxiv-MakingRisk, Ghosh17arxiv-RobustLoss} define a loss to be noise-tolerant when $\vw_\FracOut = \vw_0$ (\ie equal to the optimal weights $\vw_0$ when there are no outliers).
Several noise-tolerant losses have been proposed since then that have shown improved performance at mitigating the presence of outliers in the training data. These losses include generalized cross entropy, symmetric cross entropy, reverse cross entropy, Taylor cross entropy, among others (see Section~\ref{sec:lit-robust-dl}). While the setup has some advantages, it suffers from some limitations; for instance, one has to assume an outlier distribution to derive the noise-tolerant loss. 
As an instance of triviality that results from this, one can show that the mean square error (MSE) loss is noise-tolerant under arbitrarily severe zero-mean outliers. However, it is well known that MSE is not robust for finite sample problems (\ie $n$ in~\eqref{eq:intro-original-obj} is finite), and even one outlier can significantly degrade the corresponding estimate \citep{Huber81}.

Algorithms have been proposed to train deep learning models in the presence of outliers in the training data (see Sections~\ref{sec:lit-robust-dl} and~\ref{sec:lit-sgd-convergence}). These have been either heuristic approaches applied on the specific task of multi-label classification~\citep{Elesedy23arxiv-uclip, Li20iclr-dividemix}, or algorithms for solving a general stochastic optimization problem, albeit with outliers~\citep{Menon20iclr-clippinglabelnoise, Merad24tmlr-RobustStochastic, Chhabra24arxiv-OutlierGradient, Hu24pmlr-OutlierRobust, Shen19icml-LearningBad, Shah20pmlr-ChoosingSample, Prasad20jrss-robustEstimation}.  
These heuristics do not provide theoretical guarantees. 
The methods based on stochastic optimization analyze their convergence properties assuming an outlier distribution.
To the best of our knowledge, there are no works that analyze the region of convergence for stochastic gradient-based algorithms under arbitrary outliers.
See Section~\ref{sec:lit} for a detailed overview of the related works.

\begin{figure*}
\centering
\begin{subfigure}{.32\textwidth}
  \centering
  \includegraphics[trim={30 23 40 30},clip,width=\linewidth]{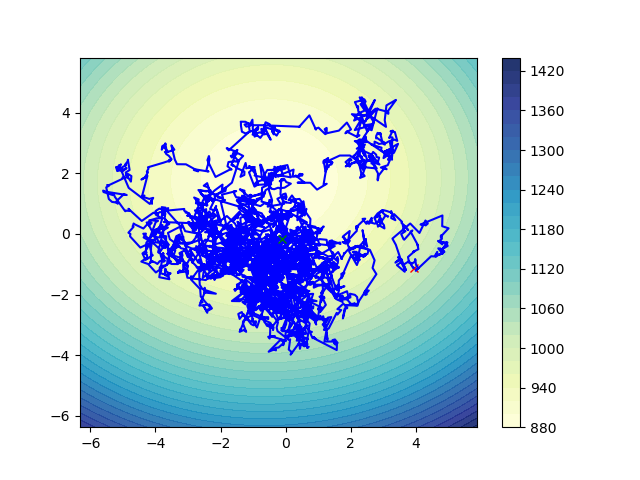}
  \caption{}
  \label{fig:sub1}
\end{subfigure}
\begin{subfigure}{.32\textwidth}
  \centering
  \includegraphics[trim={30 23 40 30},clip,width=\linewidth]{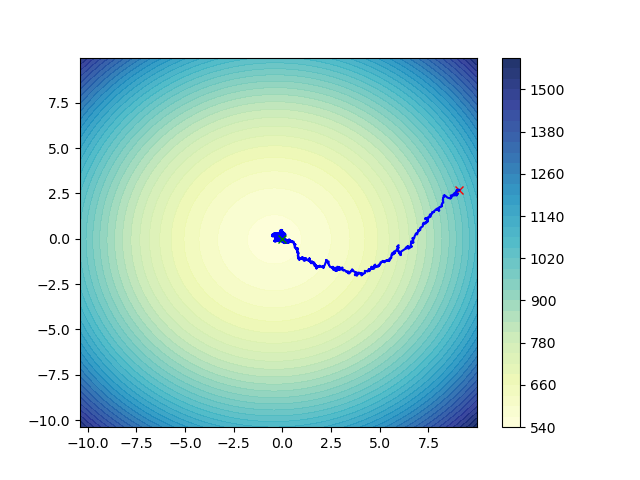}
  \caption{}
  \label{fig:sub2}
\end{subfigure}
\begin{subfigure}{.32\textwidth}
  \centering
  \includegraphics[trim={30 23 40 30},clip,width=\linewidth]{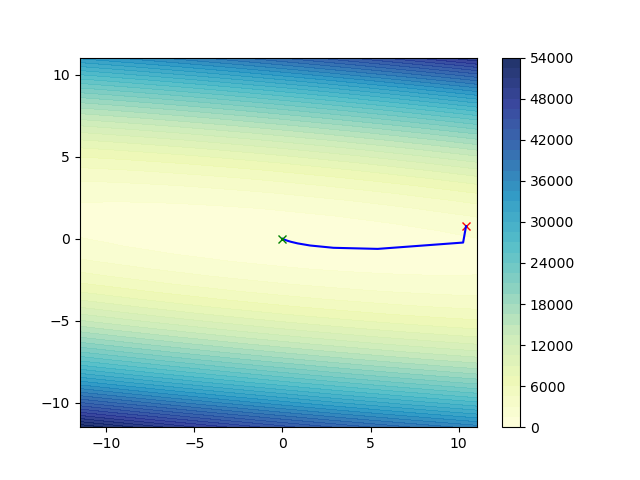}
  \caption{}
  \label{fig:sub4}
\end{subfigure}
\caption{Trajectory of (a) SGD (batch size = 1), (b) \algoNameLongCaps with Truncated Loss (batch size = 1), and (c) Gradient Descent, for a linear regression problem with zero-mean outliers. The presence of outliers in the training data introduces large perturbations into SGD. Our algorithm stabilizes the descent and the variance in the gradient estimate is lower (Lemma~\ref{lem:training-algo-variance}). We observe its behavior to be close to the full gradient descent, where the gradient estimate is exact, given zero-mean outliers.}
\label{fig:training-algo-variance}
\end{figure*}
\subsection{Contribution}
This paper makes the following key contributions:
\begin{enumerate}
	\item We expose two divergent approaches to designing robust loss functions, for training machine learning models in the presence of outliers. The first, based on a robust estimation framework, and the second, based on a risk minimization framework. We highlight that the standard Black-Rangarajan duality, proposed in the context of robust estimation, is not directly applicable to the risk-minimization setting, as it reformulates the M-estimation problem as a weighted least squares problem.
	
	\item We show that a simple modification of the Black-Rangarajan duality preserves the problem structure and makes it applicable to machine learning problems. The modification ensures that the square term $f_i^2(\vw)$ in~\eqref{eq:intro-m-est-dual} becomes linear, \ie $f_i(\vw)$. %
		Most importantly, the modified Black-Rangarajan duality gives rise to a definition of a robust loss kernel $\sigma$. We show that this robust loss kernel unifies the robust losses used in the two literatures of robust estimation and training deep classifiers.
		This enables one to now use the robust loss kernels, developed in the deep learning literature, in the robust estimation problems, and vice versa.

	\item We introduce an \emph{\algoNameLongCaps} (\algoName) based on the modified Black-Rangarajan duality. The algorithm uses a parameterized robust loss kernel $\sigma_c$ and adapts it by implementing a parameter update rule for $c$. This obviates the need for any hyper-parameter tuning. 
		We show connections between these algorithm classes and prior work. %
		We also show that the parameter update rule can be interpreted as training on conformal prediction sets, generated during training.  	

	\item We develop a convergence analysis for the \algoNameLongCaps. Under mild assumptions, we demonstrate that the robust loss kernel expands the region of convergence (compared to vanilla stochastic gradient descent) and its use reduces variance in the iterates, enabling more stable descent and improved convergence (see example in Figure~\ref{fig:training-algo-variance}).
We validate these findings experimentally on linear regression and multi-label classification tasks. We demonstrate the algorithm's efficacy in mitigating pixel-level outliers in neural scene rendering (Nerfacto; ~\cite{Tancik23siggraph-nerfstudio}), successfully recovering images with up to 80\% outliers (Figure~\ref{fig:opening-nerfacto}).
\end{enumerate}

\subsection{Organization}

The paper is organized as follows:
Section~\ref{sec:back} reviews background material and elucidates the divergent perspectives on robust loss design. 
Section~\ref{sec:prob} introduces the problem setup. 
Section~\ref{sec:url-unified} presents the modified Black-Rangarajan duality and the unified robust loss kernel $\sigma$. 
Section~\ref{sec:algo-new} provides the \algoNameLongCaps, and 
Section~\ref{sec:theory-outlier} discusses its convergence. 
Section~\ref{sec:expt} reports the experimental findings. 
Section~\ref{sec:lit} discusses related work and 
Section~\ref{sec:conclude} concludes the paper. 
All mathematical proofs are provided in appendix.

\section{Background: Diverging Principles of Robust Loss Design}
\label{sec:back}
We first review the principles that govern robust loss design in: (i) robust estimation in robotics and computer vision (Section~\ref{sec:url-robust-estimation}), and (ii) deep networks training in the presence of outliers (Section~\ref{sec:url-deep-learning}). 
We show the contrast between these two views and the need to reconcile them. 

\subsection{Robust Estimation in Robotics and Computer Vision}
\label{sec:url-robust-estimation}
Many estimation problems in robotics and computer vision can be formulated as least squares problems:
\begin{equation}
\begin{aligned}
\label{eq:nlse}
& \underset{\vw \in \setW}{\text{Minimize}}
& &\sum_{i =1}^{n} r^2_i(\vw),
\end{aligned}
\end{equation}
where $r_i(\vw)$ denotes the residual error on measurement $i$ and is typically a non-linear function of the unknown variables $\vw$. This makes solving~\eqref{eq:nlse} hard. This difficulty is exacerbated when the measurements contain outliers. In the presence of outliers, the global optima of~\eqref{eq:nlse} can deviate considerably from the ground-truth $\vw^{\ast}$. Robust estimation is used to address this issue by re-formulating~\eqref{eq:nlse} as an M-estimator:
\begin{equation}
\begin{aligned}
\label{eq:m-est}
& \underset{\vw \in \setW}{\text{Minimize}}
& &\sum_{i =1}^{n} \rho(r_i(\vw)),
\end{aligned}
\end{equation}
where $\rho$ is a robust loss function.\footnote{We use the notation $\rho$ here and keep $\sigma$ for the unified robust loss kernel defined in Section~\ref{sec:url-unified}.} Many robust losses have been proposed in the literature including the truncated least squares loss, Geman McClure loss, Welsch-Leclerc loss, Cauchy-Lorentzian loss, Charbonnier loss. \cite{Barron19cvpr-adaptRobustLoss} proposes a parameterized robust loss that recovers many robust losses upon appropriate choice of the hyperparameters. We review some common choices of robust losses below.

\vspace{4mm}
\noindent\myParagraph{Truncated Least Square loss} $\rho(r) = \min\{ r^2, c^2\}$.

\vspace{2mm}
\noindent\myParagraph{Geman McClure loss} $\rho(r) = \frac{c^2 r^2}{c^2 + r^2}$.

\vspace{2mm}
\noindent\myParagraph{Welsch-Leclerc loss} $\rho(r) = 1 - \exp(- \frac{1}{2}r^2/c^2)$.

\vspace{2mm}
\noindent\myParagraph{Cauchy-Lorentzian loss} $\rho(r) = \log(1 + \frac{1}{2}r^2/c^2)$.

\vspace{2mm}
\noindent\myParagraph{Charbonnier loss} $\rho(r) = \sqrt{r^2/c^2 + 1} - 1$.

\vspace{2mm}
\noindent\myParagraph{Barron's loss} $\rho(r) = \frac{|\alpha - 2|}{\alpha} \left(  \left( \frac{(x/c)^2}{|\alpha - 2|} + 1 \right)^{\alpha/2} - 1 \right)$

\vspace{4mm}
\noindent The robust estimation problem~\eqref{eq:m-est} is often solved by re-writing it as a weighted least squares problem: 
\begin{equation}
\begin{aligned}
\label{eq:weighted-nlse}
& \underset{\vw \in \setW, u_i \in [0, 1]}{\text{Minimize}}
& &\sum_{i =1}^{n} u_i r^2_i(\vw) + \Psi_{\rho}(u_i),
\end{aligned}
\end{equation}
where $\Psi_\rho$ is an outlier process (\ie a function that depends on the choice of the robust loss function $\rho$). 
The Black-Rangarajan duality \citep{Black96ijcv-unification} shows the equivalence between the M-estimator~\eqref{eq:m-est} and the weighted non-linear least squares problem~\eqref{eq:weighted-nlse} for suitable choices of $\Psi_\rho$. 
\begin{theorem}[\cite{Black96ijcv-unification}]
\label{thm:br}
	The robust estimation problem~\eqref{eq:m-est} is equivalent to the weighted non-linear least squares problem~\eqref{eq:weighted-nlse} with $\Psi_{\rho}(u) = - u (\phi')^{-1}(u) + \phi( (\phi')^{-1}(u))$ and $\phi(r) = \rho(\sqrt{r})$, provided $\phi(r)$ satisfies: (i) $\phi'(r) \rightarrow 1$ as $r \downarrow 0$, (ii) $\phi'(r) \rightarrow 0$ as $r \uparrow +\infty$, and (iii) $\phi''(r) < 0$.
	
\end{theorem}
The Black-Rangarajan duality motivates solving~\eqref{eq:m-est} by iteratively minimizing the weighted non-linear least squares problem. The coefficient weights $u_i = \rho(r_i(\vw))/ 2 r_i(\vw)$ are chosen using $\vw$ from the previous iteration \citep{Black96ijcv-unification}. This has been leveraged to develop robust algorithms for estimation problems in computer vision and robotics (\eg~\cite{Yang20ral-GNC, Chebrolu20arxiv-adaptiveCost, Peng23cvpr-ConvergenceIRLS}).

\subsection{Training Deep Learning Models in the Presence of Outliers} 
\label{sec:url-deep-learning}
The risk-minimization framework suggests that a deep learning model is trained to obtain the model weights:
\begin{equation}
	\vw^{\ast} = \underset{\vw \in \setW}{\arg\min}~~\E_{(\vxx, \vy) \sim \calD}[l(\vg(\vw, \vxx), \vy)],
\end{equation}
where $l: \setY\times \setY \rightarrow \reals_{+}$ is a loss function, $\setY$ denotes the set of all outputs, $\calD$ denotes the distribution of pairs $(\vxx, \vy)$ when there are no outliers, and $\vg(\vw, \vxx)$ is the model that predicts output $\vy$, given input $\vxx$ and model weights $\vw$.
The goal of robust loss $l$ design should be such that $\vw^{\ast}$ does not change much when we introduce outliers in the distribution $\calD$.
Let 
\begin{equation}
	\vw^{\ast}_{\lambda} = \underset{\vw \in \setW}{\arg\min}~~\E_{(\vxx, \vy) \sim \calD_{\lambda}}[l(\vg(\vw, \vxx), \vy)],
\end{equation}
denote the optimal model weights when the dataset contains $\FracOut$ fraction of outliers; here,  
training data now comes from an outlier-contaminated distribution $\calD_{\lambda}$, where $\lambda$ fraction of data are outliers. A loss function is said to be \emph{noise-tolerant} at noise rate $\lambda$ if $\vw^{\ast} = \vw^{\ast}_{\lambda}$.

\cite{Ghosh15arxiv-MakingRisk, Ghosh17arxiv-RobustLoss} show that the classical cross-entropy (CE) loss is not noise-tolerant for the task of classification. These works further show that a simple mean absolute error (MAE) loss is noise-tolerant to any $\lambda < 1 - 1/K$ fraction of outliers, where $K$ denotes the total number of classes. 
Several noise-tolerant losses have been proposed since then, including generalized cross-entropy (GEC) loss, symmetric cross-entropy loss (SCE), finite Taylor series expansion of log likelihood loss, and asymmetric losses (see Section~\ref{sec:lit-robust-dl}). Let $\vp = \vg(\vw, \vxx)$ and $\vp[y]$ denote the predicted probability of class label $y$,\footnote{Intuitively, for a $K$-class classification problem, $\vp = \vg(\vw, \vxx) \in [0,1]^{K}$ is the vector of probabilities assigned by the model to each class, and $\vp[y]$ is the probability assigned to the ground-truth class $y$.}
 then the following are commonly adopted losses (with constants $a, A \in \reals$ and $p, q$ positive integers):

\vspace{4mm}
\noindent\myParagraph{Mean absolute error (MAE)} $l(\vp, y) = 1 - \vp[y].$

\vspace{2mm}
\noindent\myParagraph{Generalized cross-entropy (GCE)} $l(\vp, y) = \frac{1}{q}(1 - \vp[y]^q).$

\vspace{2mm}
\noindent\myParagraph{Symmetric cross-entropy (SCE)} $l(\vp, y) = - \log ( \vp[y] ) - A \sum_{k \neq y} \vp[k].$

\vspace{2mm}
\noindent\myParagraph{Reverse cross-entropy (RCE)} $l(\vp, y) = - A \sum_{k \neq y} \vp[k].$

\vspace{2mm}
\noindent\myParagraph{Taylor cross-entropy (t-CE)} $l(\vp, y) = \sum_{m=1}^{t} \frac{1}{m}( 1 - \vp[y])^m.$

\vspace{2mm}
\noindent\myParagraph{Asymmetric generalized cross-entropy (AGCE)} $l(\vp, y) = \frac{1}{q}\left( (a+1)^q - (a + \vp[y])^q \right).$

\vspace{2mm}
\noindent\myParagraph{Asymmetric unhinged loss (AUL)} $l(\vp, y) = \frac{1}{p}\left(  (a - \vp[y])^p - (a-1)^p \right).$

\vspace{2mm}
\noindent\myParagraph{Asymmetric exponential loss (AEL)} $l(\vp, y) = \exp\left( - \vp[y]/a\right).$

\vspace{4mm}
\noindent All these losses are up to a constant away from the original definitions in the literature \citep{Zhang18nips-GeneralizedCross, Amid19nips-RobustBiTempered, Wang19iccv-SymmetricCross, Feng20ijcai-CanCross, Zhou23pami-AsymmetricLoss}. We take this liberty because a constant factor does not affect the optima. Note that these losses can be written down as
\begin{equation}
\label{eq:robustifying}
	l(\vp, y) = \rho(-\log \vp[y]), 
\end{equation}
where $-\log \vp[y]$ is the standard cross-entropy loss (Appendix~\ref{app:classification-robust-losses}). 
This implies that we can construct a robust loss kernel $\rho$ for each of these losses with respect to the standard cross-entropy loss. However, several problems arise in articulating the above losses this way.
Firstly, a direct application of Black-Rangarajan duality yields an equivalence between the robust estimation problem as the squared cross-entropy loss:
\begin{equation}
\begin{aligned}
\label{eq:sq-ce}
& \underset{\vw \in \setW, u_i \in [0, 1]}{\text{Minimize}}
& &\sum_{i =1}^{n} u_i \left( \log(\vp[y]) \right)^2 + \Psi_{\rho}(u_i),
\end{aligned}
\end{equation}
which does not make much sense. %
Ideally, we would like a duality result where the robust estimation problem is shown to be equivalent to a weighted cross-entropy minimization problem in~\eqref{eq:sq-ce}. In general, the dual of the robust estimation should be a weighted version of the original problem. The weights should indicate the confidence in the sample being an outlier. 
We would then be able to apply this result to non-linear least squares, as well as cross-entropy minimization. 

The second problem that arises in using the dual~\eqref{eq:sq-ce} is that many of the robust loss kernels $\rho$ (see~\eqref{eq:robustifying}) used in this conversion do not satisfy the requirement on $\rho$ in Theorem~\ref{thm:br}. 

\begin{remark}[Risk Minimization Framework and Robust Losses]
The risk minimization framework used to define robust losses (\ie noise-tolerant losses) is limiting. The framework ignores the effects of finite sample size and forces the designer to make unrealistic assumptions about the outlier distribution. For example, it can be easily shown that the mean squared error (MSE) is noise-tolerant for regression problems, if the outliers are zero mean. However, fragility of MSE to outliers in finite sample problems is well documented \citep{Huber81}.
Figure~\ref{fig:training-algo-variance} shows that even when training using the MSE loss (which is noise-tolerant) using SGD results in convergence issues due to high variance in gradient estimates.
\end{remark}

\begin{remark}[Convergence and Robust Loss Design]
	While many works have proposed robust losses for training deep learning models, there has been little effort at understanding the effect of robust losses on convergence in the presence of outliers. Analytical results relating the structure of the robust loss to the region of convergence and outlier mitigation are unknown. A lack of any structure on the robust losses $\rho$ means that researchers have to heuristically experiment on a wide variety of datasets before being confident in the resulting model.
\end{remark}

\section{Problem Statement}
\label{sec:prob}

We are given a dataset of $n$ samples, potentially corrupted by outliers. Each sample is an input-output pair $(\vxx_i, \vy_i)$.  
The goal is to train a model while mitigating the presence of outliers in the dataset. The model is parameterized by model weights $\vw \in \Real{d}$. 
In the outlier-free case, each measurement $i$ is associated with a loss $f_i(\vw) = l(\vh(\vw, \vxx_i), \vy_i) \geq 0$, and the model is trained to solve the following optimization problem:
\begin{equation}
\begin{aligned}
\label{eq:objective}
& \underset{\vw \in \setW}{\text{Minimize}}
& &f(\vw) = \frac{1}{n}\sum_{i=1}^{n} f_i(\vw).
\end{aligned}
\end{equation}
When the measurements are contaminated by outliers, we would ideally like to minimize the following, outlier-free objective, instead of~\eqref{eq:objective}:
\begin{equation}
\begin{aligned}
\label{eq:inliers-objective}
& \underset{\vw \in \setW}{\text{Minimize}}
& &f_{I}(\vw) = \frac{1}{n}\sum_{i=1}^{n} f_{i, I}(\vw),
\end{aligned}
\end{equation}
where $f_{i, I}(\vw)$ denotes the outlier-free component of the loss, \ie $f_{i, I}(\vw) = f_i(\vw)$ for inliers and zero otherwise.
However, it is not possible to know $ f_{i, I}(\vw)$ in practice and we are constrained to work with $f(\vw)$ and $f_i(\vw)$, while attempting to minimize~\eqref{eq:inliers-objective}.
Let $f^{\ast}_{I}$ denote the optimal value of~\eqref{eq:inliers-objective}. The goal is to find $\hat{\vw}$ such that $f_{I}(\hat{\vw})$ is as close to $f^{\ast}_{I}$ as possible, and we need to do this using only $f_i(\vw)$ and $f(\vw)$ in~\eqref{eq:objective}.
We do not know $\Nout$, the number of outliers, and assume that they are arbitrary, and do not follow a specific distribution. 
We use $\Nin = n - \Nout$ and $\FracOut = \Nout/n$ to denote the number of inliers and the fraction of outlier measurements, respectively.

In the next section, we bring out a unified definition of a robust loss kernel $\sigma$ (Definition~\ref{def:robust-loss-kernel}) based on a simple modification of the Black-Rangarajan duality (Corollary~\ref{cor:br-linear}).
In Section~\ref{sec:algo-new}, we make use of the modified Black-Rangarajan duality to propose an \algoNameLongCaps, for training deep learning models in the presence of arbitrary outliers. We prove convergence properties of the algorithm in Section~\ref{sec:theory-outlier}.

\section{Unified Robust Loss Kernel}
\label{sec:url}
\begin{table*}
\vspace{-1cm}
\centering
\caption{Robust loss kernels that correspond to popular robust losses in robotics and computer vision (Section~\ref{sec:url-robust-estimation}) and in training deep learning models  (Section~\ref{sec:url-deep-learning}).}
\begin{tabular}{lc}
\vspace{3mm}
\textbf{Robust Loss Kernel} & $\sigma(r)$ \\  
\hline  \\ 
\vspace{4mm}
Linear Truncated Kernel & $c \cdot \min \{ r/c, 1\}$ \\

\vspace{4mm}
Geman McClure Kernel & $c \cdot \frac{r/c}{1 + r/c}$ \\

\vspace{4mm}
Welsch-Leclerc Kernel & $c \cdot (1 - \exp(- r/c) )$ \\

\vspace{4mm}
Cauchy-Lorentzian Kernel & $c \cdot \log(1 + r/c)$ \\

\vspace{4mm}
Charbonnier Kernel & $2c \cdot \sqrt{r/c + 1} - 1$ \\

\vspace{4mm}
Barron's Kernel & \hspace{3mm}$c \cdot \frac{|\alpha - 2|}{\alpha} \left( \left( \frac{r/c}{|\alpha - 2|} + 1\right)^{\alpha/2} - 1\right)$ \\

\vspace{4mm}
Mean error kernel & $1 - \exp(-r)$ \\

\vspace{4mm}
Generalized cross-entropy kernel & $\frac{1}{q}(1 - \exp(-qr))$ \\

\vspace{4mm}
Symmetric cross-entropy kernel & $\frac{1}{1 + A}(r + A \exp(-r) - A)$ \\

\vspace{4mm}
Taylor cross-entropy kernel & $\sum_{m=1}^{t}\frac{1}{m}(1 - \exp(-r))^m$ \\

\vspace{4mm}
Asymmetric generalized cross-entropy kernel & \hspace{4mm}$\frac{1}{q \cdot a^{q-1}}\left( (a+1)^q - (a + \exp(-r))^q \right)$ \\

\vspace{4mm}
Asymmetric unhinged kernel & \hspace{4mm}$\frac{1}{p \cdot a^{p-1}} \left( (a - \exp(-r))^p - (a -1)^p \right)$ \\

\vspace{4mm}
Asymmetric exponential loss & $a \cdot \exp\left( \frac{1}{a}(1 - \exp(-r))\right)$ \\
\hline
	
\end{tabular}
\label{tab:robust-loss-kernels}
\vspace{-1mm}
\end{table*}
\label{sec:url-unified}
We now present a unified framework that bridges the formulations introduced in Section~\ref{sec:back}. We first prove a modified version of the Black-Rangarajan duality. This version helps keep the problem structure intact. That is, the dual of a robust cross-entropy minimization problem is a weighted cross-entropy minimization problem. Similarly, 
the dual of a robust non-linear least squares estimation problem is a weighted non-linear least squares problem. In both cases, weights indicate the confidence that the measurement is an inlier (\ie the higher the weight, the greater the confidence that the measurement is an inlier). 
The modified dual also gives rise to a definition of a robust loss kernel $\sigma$ that is simple and intuitive. We will see that all the robust losses we have seen in Sections~\ref{sec:url-robust-estimation}-\ref{sec:url-deep-learning} can be modified to meet this definition. %

\subsection{Modified Black-Rangarajan Duality}

We state and prove a simple modification of the Black-Rangarajan Duality.  %
\begin{corollary}[Modified Black-Rangarajan Duality]
\label{cor:br-linear}
	The robust estimation problem, 
	\begin{equation}
	\label{eq:br-linear-primal}
		\begin{aligned}
			\underset{\vw, u_i \in [0, 1]}{\text{Minimize}} & & \frac{1}{n}\sum_{i=1}^{n} \sigma(f_i(\vw)),
		\end{aligned}
	\end{equation}
	with robust loss kernel $\sigma(\cdot)$ is equivalent to 
	\begin{equation}
	\label{eq:br-linear-dual}
		\begin{aligned}
			\underset{\vw, u_i \in [0, 1]}{\text{Minimize}} & & \frac{1}{n}\sum_{i=1}^{n} \left[ u_i \cdot  f_i(\vw)  + \Phi_{\sigma}(u_i) \right],
		\end{aligned}
	\end{equation}
	where $\Phi_{\sigma}(u) = - u (\sigma')^{-1}(u) + \sigma( (\sigma')^{-1}(u))$, provided $\sigma(r)$ satisfies: (i) $\sigma'(r) \rightarrow 1$ as $r \downarrow 0$, (ii) $\sigma'(r) \rightarrow 0$ as $r \uparrow +\infty$, and (iii) $\sigma''(r) < 0$.
\end{corollary}
\begin{proof}
	The proof is obtained by substituting $\sigma(r) = \rho(\sqrt{r})$ and $r_i(\vw)^2 = f_i(\vw)$ in the Black-Rangarajan duality (Theorem~\ref{thm:br}).
	We also provide a proof from first principles in Appendix~\ref{pf:cor:br-linear}.
\end{proof}
\begin{remark}[Dual Problem Structure and its Application]
	The modified Black-Rangarajan duality keeps the problem structure intact, \ie the dual problem minimizes a sum of weighted losses $f_i(\vw)$. This is in contrast with the original Black-Rangarajan dual where the dual problem would have been to minimize the sum of weighted squares $f^2_i(\vw)$. This allows us to apply the modified Black-Rangarajan duality to train deep learning models in the presence of outliers (see Section~\ref{sec:algo-new}).
\end{remark}

\subsection{Unified Robust Loss Kernel}
The modified Black-Rangarajan duality imposes constraints on $\sigma$. We inspect these constraints and see that they provide a simple and intuitive definition of a \emph{robust loss kernel} that can be applied generally across all deep learning problems. 
The modified duality (Corollary~\ref{cor:br-linear}) requires: 
\begin{enumerate}
	\item[C1:] $\sigma'(r) \rightarrow 1$ and $r \downarrow 0$.
	\item[C2:] $\sigma'(r) \rightarrow 0$ and $r \uparrow +\infty$. 
\end{enumerate}
This indicates that $\sigma$ should be such that for small $r$ it behaves like a linear function, \ie $\sigma(r) \approx r$ for $r$ close to $0$. For large $r$, on the other hand, $\sigma$ behaves like a constant function. As a robust loss kernel, for smaller loss terms, it leaves the original losses unaltered, while for larger loss terms, it damps their effect on the total loss. The third condition:
\begin{enumerate}
 	\item[C3:] $\sigma''(r) < 0$,
\end{enumerate}
implies that $\sigma'(r)$ is a monotonically decreasing function and $\sigma'(r) \in [0, 1]$ for all $r$. A consequence of this is that $\sigma$ is a monotonically increasing function, and therefore, preserves ordering of the losses (\ie $f_i(\vw) \leq f_j(\vw)$ implies $\sigma(f_i(\vw)) \leq \sigma(f_j(\vw))$). All this makes for a simple, intuitive, and verifiable definition of a robust loss kernel $\sigma$:

\begin{definition}[Robust Loss Kernel $\sigma$] 
\label{def:robust-loss-kernel}
A function $\sigma: \reals \rightarrow \reals$ is a robust loss function if (i)  $\sigma'(r) \rightarrow 1$ as $r \downarrow 0$, (ii) $\sigma'(r) \rightarrow 0$ as $r \uparrow +\infty$, and (iii) $\sigma''(r) \leq 0$.
\end{definition}
We relax the strict concavity of $\sigma$ in C3 to the condition $\sigma''(r) \leq 0$. The strict concavity is required for the modified Black-Rangarajan duality to hold (in particular, to ensure invertibility of $\sigma'$). A truncated kernel $\sigma(r) = c\min\{ r/c, 1\}$ does not have an invertible $\sigma'$ (and does not satisfy C3), but can still be a valid robust loss kernel according to Definition~\ref{def:robust-loss-kernel}.

Table~\ref{tab:robust-loss-kernels} presents various robust loss kernels. The first six robust kernels are derived from common robust losses used in the robotics and computer vision literature (see Section~\ref{sec:url-robust-estimation}). The next eight kernels are derived from the robust losses used in training deep learning-based classifier models (see Section~\ref{sec:url-deep-learning}). It can be verified that each robust loss kernel corresponds to a robust loss presented in Sections~\ref{sec:url-robust-estimation}-\ref{sec:url-deep-learning} (see Appendix~\ref{app:robust-losses-new}).
All the kernels presented in Table~\ref{tab:robust-loss-kernels} satisfy Definition~\ref{def:robust-loss-kernel}, and can be applied with the modified Black-Rangarajan duality to various machine learning problems.

\section{\algoNameLongCaps}
\label{sec:algo-new}
The modified Black-Rangarajan duality (Corollary~\ref{cor:br-linear}) 
motivates our \emph{\algoNameLongCaps} (\algoName) to solve problem~\eqref{eq:br-linear-primal}. 
The key idea is to solve the dual (\ie~\eqref{eq:br-linear-dual}) using an alternation algorithm. The alternation algorithm first optimizes  the model weights $\vw$ given $\vu$ using gradient-based minimization, and then optimizes  $\vu$ given $\vw$ (Section~\ref{sec:algo-gnc}). We use a robust loss kernel $\sigma_c$, parameterized by a constant $c$ (\cf with Table~\ref{tab:robust-loss-kernels}). We use this parameter 
to give the algorithm extra flexibility, by allowing it to \emph{adapt} $\sigma_c$ as the training progresses. In particular, we update $c$ as the training iterations progress (Section~\ref{sec:algo-param-update}).
The \algoName algorithm with $T$ iterations of a gradient-based optimizer is given in Algorithm~\ref{algo:athree}.

\begin{algorithm}
\caption{\algoNameLongCaps (\algoNameT)}\label{alg:cap}
\label{algo:athree}
\begin{algorithmic}
\State \textbf{Input:} (i) model weights $\vw_0$, (ii) coefficient weights $u_{i, 0}$, (iii) parameter $c_0$.
\State \textbf{Set:} $t \leftarrow 0$.
\For{$t=0,\ldots,\text{max. number of iterations}$}

\vspace{3mm}
\State \underline{Parameter Update (Section~\ref{sec:algo-param-update}):}
\vspace{1mm}
\If{Parameter Update at $t$} %
	\State Solve~\eqref{eq:gnc-parameter-update} to obtain $c_{t}$.
\Else
	\State $c_{t} \leftarrow c_{t-1}$.
\EndIf

\vspace{3mm}
\State \underline{Model Weight Update (Section~\ref{sec:algo-gnc}):}  
\vspace{1mm}
\State Solve~\eqref{eq:gnc-weight-update-final} using $T$ iterations of a gradient-based algorithm. Obtain $\vw_{t+1}$. 

\vspace{3mm}
\If{Stopping Criteria Are Satisfied} %
	\State Break.
\EndIf
\EndFor
\end{algorithmic}
\end{algorithm}

\subsection{Alternation Algorithm}
\label{sec:algo-gnc}
We describe the alternating minimization strategy that updates the model weights $\vw$ and the coefficient weights $\vu$. 
We will assume the robust loss parameter $c$ to be fixed in this section for ease of presentation.

Let $\vw_0$ and $u_{i, 0}$ be the initial model and coefficient weights. Applying block coordinate descent to the modified dual~(eq.\eqref{eq:br-linear-dual}, Corollary~\eqref{cor:br-linear})
we derive coefficient and weight update steps as
	\begin{equation}
	\label{eq:gnc-coefficient-update}
		u_{i, t} = \underset{u \in [0, 1]}{\text{ArgMinimize}}~~  u \cdot f_i(\vw_t) + \Phi_{\sigma_{c}}(u),
	\end{equation}
	and 
	\begin{equation}
	\label{eq:gnc-weight-update}
	\vw_{t+1} = \underset{\vw}{\text{ArgMinimize}}~\frac{1}{n}\sum_{i=1}^{n} u_{i, t} \cdot f_i(\vw),
	\end{equation}
respectively.
The weight update step~\eqref{eq:gnc-weight-update} can be performed in many ways. It can involve running any existing gradient-based algorithm (\eg SGD, ADAM) either to convergence or running it for a few iterations. 
The coefficient weight update~\eqref{eq:gnc-coefficient-update} has a simple analytical solution.

\begin{lemma}
	\label{lem:analytical-coefficient-update}
	The coefficient weight $u^{\ast}$ that solves $\underset{u \in [0, 1]}{\arg\min}~u \cdot f_i(\vw) + \Phi_{\sigma_c}(u)$ is given by 
	$u^{\ast} = \sigma'_c(f_i(\vw))$.
\end{lemma}
The proof is given in Appendix~\ref{pf:lem:analytical-coefficient-update}.
This simplifies the model weight update step to
\begin{equation}	
	\label{eq:gnc-weight-update-final}
	\vw_{t+1} = \underset{\vw}{\text{ArgMinimize}} ~~\frac{1}{n} \sum_{i=1}^{n} \sigma'_{c}(f_i(\vw_t)) \cdot f_i(\vw).
\end{equation}
Note that the coefficient weights $u_{i, t} = \sigma'_{c_t}(f_i(\vw_t))$ remain fixed and is determined by $\vw_t$.

\subsection{Parameter Update}
\label{sec:algo-param-update}
We now describe the update rule for the parameter $c$ of the robust loss kernel $\sigma_c$.
\begin{algo}[Parameter Update Rule]
	Given model weights $\vw_t$ and the parameterized robust loss kernel $\sigma_c$, the parameter update $c_{t}$ is computed by solving
	\begin{equation}
		\label{eq:gnc-parameter-update}
		c_{t} = \underset{c \in [0, 1]}{\text{Find}} \left\{ \frac{1}{|\calD|}\sum_{i \in \calD} \sigma'_c(f_i(\vw_t)) = \zeta \right\},
	\end{equation}
	where $\calD$ denotes the set of all accumulated measurements across previous iterations (\ie with $c=c_{t-1}$) and $\zeta$ is a positive constant.
\end{algo}
The rational for this rule is as follows. In the dual problem~\eqref{eq:gnc-weight-update-final}, the coefficient weights $\sigma_c'(f_i(\vw_t))$ can be interpreted as the likelihood that the measurement $i$ is an inlier, \ie $\prob{i \in \Nin~|~\vw_t, \vc}$. However, we note that these probabilities should satisfy a constraint. There are a fixed number of outliers and inliers, respectively; \ie $\sum_{i =1}^{n} \identity{i \in \Nin} = \Nin$. Taking conditional expectation on both sides (\wrt $\vw_t, \vc$) we obtain $\sum_{i=1}^{n} \prob{i \in \Nin~|~\vw_t, \vc} = \Nin$, which implies 
\begin{equation}
	\label{eq:frac_in}
	\frac{1}{n}\sum_{i=1}^{n} \sigma_c'(f_i(\vw_t)) = \frac{\Nin}{n}.
\end{equation}
Thus, the average of all the $\sigma_c'(f_i(\vw_t))$ must be a constant. In fact, we know that it should equal the fraction of outliers in the training data. We impose this constraint to obtain our parameter update. We can tune $\zeta$ as a hyper-parameter. We implement~\eqref{eq:gnc-parameter-update} using a simple binary search algorithm. 

In the \algoNameLongCaps, the robust loss parameter $c_t$ is updated every few iterations according to the update rule~\eqref{eq:gnc-parameter-update} (Algorithm~\ref{algo:athree}).  
We make some remarks about how our \algoNameLong relates to prior work in the literature.

\begin{remark}[Parameter Update and Graduated Non-Convexity]
	Graduated Non-Convexity (GNC) is a popular approach for robust estimation in robotics and vision~\citep{Black96ijcv-unification,Blake1987book-visualReconstruction,Yang20ral-GNC, Chebrolu20arxiv-adaptiveCost, Peng23cvpr-ConvergenceIRLS}. GNC solves the M-estimation problem by utilizing the 
original Black-Rangarajan duality 
		and defining a surrogate loss, parameterized by $\mu$. The parameter $\mu$ is updated during training to enhance convergence.
While showing good performance, GNC requires careful hyper-parameter tuning, which is difficult in some applications~\citep{Chebrolu20ral-adaptiveCost}.	
On the other hand, our parameter update rule avoids constructing any auxiliary loss and adapts the robust loss parameter $c$ directly. This obviates the need to treat $c$ as a separate hyper-parameter. The update rule~\eqref{eq:gnc-parameter-update}, however, results in another hyper-parameter: $\zeta$. This, it turns out, is much easier and intuitive to tune as it relates to the fraction of inliers expected in the dataset (see~\eqref{eq:frac_in}). %
\end{remark}

\begin{remark}[Iteratively Trimmed Loss Minimization] 
\label{rem:gnc-itlm}
The parameter update~\eqref{eq:gnc-parameter-update} updates the robust loss kernel $\sigma_c$ enabling it to better separate between inliers and outliers. With this parameter update, the 
\algoNameLong
can be viewed as a generalization of the iteratively trimmed loss minimization by \cite{Shen19icml-LearningBad}. \cite{Shen19icml-LearningBad} train on the best $\alpha \cdot n$ measurements (here, best implies measurements with the lowest loss $f_i(\vw_t)$ and $\alpha$ is a hyper-parameter). When $\sigma_c(r) = c \cdot \max\{ r/c, 1\}$ the update rule~\eqref{eq:gnc-parameter-update} becomes 
	\begin{equation}
		c_{t} = \text{Find}_{c \in [0, 1]} \left\{ \frac{1}{|\calD|}\sum_{i \in \calD} \mathbb{I}\{f_i(\vw_t) \leq c\} = \zeta\right\}, \label{eq:rem-parameter-update}
	\end{equation}
	\ie it selects $c_{t}$ such that $\zeta \cdot n$ best samples are used in training. Our differentially continuous robust kernel $\sigma_c$ generalizes this rule. %
\end{remark}

\begin{remark}[Iteratively Training with Conformal Set Prediction] 
\label{rem:gnc-cp}
Note that~\eqref{eq:rem-parameter-update}, in fact, generates a conformal prediction set given a quantile $\zeta$ \citep{Shafer08jmlr-TutorialConformal}. The set $\setC_{t} = \left\{ i \in [n]~|~f_i(\vw_t) \leq c_{t}\right\}$ is the predicted set of good samples that fall within the $\zeta$ quantile. Using $\sigma_c(r) = c \cdot \min\{ r/c, 1\}$, therefore, results in an algorithm where one computes a conformal prediction set of samples, and trains on them. The process iterates till convergence. %
This observation shows an interesting connection and a promise of using uncertainty quantification methods for outlier rejection and self-training of machine learning models.
\end{remark}
In the next section, we analyze convergence 
of the \algoNameLong to the outlier-free optima.

\section{Theoretical Analysis}
\label{sec:theory-outlier}
We now analyze the convergence of \algoNameT (Algorithm~\ref{algo:athree}), which uses $T$ iterations of stochastic gradient descent (SGD) as a gradient-based solver for the weight update~\eqref{eq:gnc-weight-update-final}.
We analyze the convergence behavior of \algoNameO (\ie \algoNameT with $T=1$) and extend it to \algoNameT.
In particular, we show that \algoNameT reduces the variance in the gradient computation (Section~\ref{sec:theory-variance}) and increases the region of convergence (\ie convergence to $f^{\ast}_I$), in the presence of outliers (Section~\ref{sec:theory-roi}). 

\subsection{Assumption on Outliers}
\label{sec:theory-assumptions}
We first make a few assumptions about how the outliers impact the outlier-free objective $f_{I}(\vw)$. It turns out that we do not require an explicit relation between the loss component $f_i(\vw)$ and its outlier-free version $f_{i, I}(\vw)$. Our adaptive alternation algorithm is gradient-based, and therefore, we only require assumption about how the outlier $\vo_i$ impacts the gradient $\nabla f_i(\vw)$. We assume that the outliers perturb the true gradient $\nabla f_{i, I}(\vw)$ in an additive manner. 
\begin{assumption}[Outlier Gradient]
\label{as:outlier-gradient}
	For outlier measurements $i \in \Nout$ we have  
	\begin{equation}
		\nabla f_i(\vw) = \nabla f_{i, I}(\vw) + \vh_i(\vo_i, \vw),
	\end{equation}
	where $\vh_i(\vo_i, \vw) \in \Real{d}$ and is unknown.
\end{assumption}
We verify that this assumption holds for two broad class of problems, namely non-linear regression and multi-label classification in Appendix~\ref{app:as:outlier-gradient}.

\begin{remark}[Huber Contamination Model]
	We remark here that Assumption~\ref{as:outlier-gradient} is different from the Huber contamination model considered in related works, \eg~\citep{Merad24tmlr-RobustStochastic, Prasad20jrss-robustEstimation}. In the Huber contamination model, the outlier factor $\vh_i(\vo_i, \vw)$ does not depend on $i$ and is assumed to follow a distribution. In our study, $\vh_i(\vo_i, \vw)$ not only does not depend on $i$, but is also arbitrary. 
\end{remark}

\noindent We next make a final assumption to make things analytically easier. 
\begin{assumption}[Low Signal-to-Outlier Ratio]
\label{as:low-signal-to-outlier}
The outlier noise is large and is larger than its signal, \ie $\norm{\vh_i(\vo_i, \vw)} \geq 1$, $\norm{\vh_i(\vo_i, \vw)} \geq \norm{\nabla f_{i, I}(\vw)}$ for all $i \in \Nout$. 
\end{assumption}

\subsection{Variance in Updates}
\label{sec:theory-variance}
Outliers in the dataset can affect the computed gradients $\vg_t$ and render the algorithm  unstable and not convergent to optima (see Figure~\ref{fig:training-algo-variance}). The loss function plays a key role in determining how the outliers affect the gradients (see Examples~\ref{ex:non-linear-regression} and~\ref{ex:classification} in Appendix~\ref{app:as:outlier-gradient}). We next show how \algoNameO is able to control the variance of the descent direction better. We consider batch size of one for ease of presentation.
\begin{lemma}
\label{lem:training-algo-variance}
	Consider batch size of one in training algorithms and assume the outliers to be zero mean, \ie $\frac{1}{\Nout} \sum_{i \in \Nout} \vh_i(\vo_i, \vw) = 0$. The variance in the descent direction, \ie $\E_i [\norm{\vg_t - \nabla f_I(\vw)}^2]$, for the SGD and \algoNameO is given by 
		\begin{equation}
			3 \eta^2 \FracOut \frac{1}{\Nout}\sum_{i=1}^{\Nout} \norm{\vh_i(\vo_i, \vw_t)}^2,
		\end{equation}
		and
		\begin{equation}
		3 \eta^2 \FracOut \frac{1}{\Nout}\sum_{i=1}^{\Nout} \sigma'_c (f_i(\vw))^2 \norm{\vh_i(\vo_i, \vw_t)}^2,
		\end{equation}
		respectively.
\end{lemma}

\begin{remark}[Robust Loss Kernel's Derivative $\sigma'_c$]
	We see here that the presence of the coefficient weight $\sigma'_c(f_i(\vw))^2$ helps control the variance. Observe that $\sigma'_c(f_i(\vw))$ tends to be small for outliers. If it is inversely proportional to $\norm{\vh_i(\vo_i, \vw)}$ then the outlier variance can be greatly mitigated. We observe this phenomena in experiments. Figure~\ref{fig:training-algo-variance} shows the impact for the case of linear regression. Note that this insight is missed when using the notion of noise-tolerant losses (see Section~\ref{sec:url-deep-learning}).
\end{remark}
In the next subsection, we will see how the same variance bound determines the region of convergence for our adaptive alternation algorithm. %

\subsection{Increased Region of Convergence}
\label{sec:theory-roi}
We now analyze convergence  
for the \algoNameLongCaps.
We also derive convergence results for stochastic gradient descent as they serve as a good comparison. Our goal is to discuss converge to the outlier-free optima $f^{\ast}_{I}$, rather than the global optima of a robust estimation problem. We make two structural assumptions on the outlier-free objectives, \ie $f_{i, I}(\vw)$. We assume them to be $L$-smooth and $\mu$-Polyak-Lojasiewicz. 
\begin{definition}[$L$-smooth]
A continuously differentiable function $f$ is said to be $L$-smooth if it satisfies
\begin{equation}
	f(\vy) \leq f(\vxx) + \nabla f(\vxx)\tran (\vy - \vxx) + \frac{L}{2}\norm{\vy - \vxx}^2.
\end{equation}
\end{definition}
	
\begin{definition}[$\mu$-Polyak-Lojasiewicz]
A continuously differentiable function $f$ is said to be $\mu$-Polyak-Lojasiewicz if 
\begin{equation}
	f(\vw) - \min_{\vw} f(\vw) \leq \frac{1}{2\mu} \norm{\nabla f(\vw)}^2.
\end{equation}
\end{definition}
We remark that if the $f_i$'s are all $L$-smooth or $\mu$-Polyak-Lojasiewicz, then $f = \frac{1}{n} \sum_{i=1}^{n} f_i$ is also $L$-smooth and $\mu$-Polyak-Lojasiewicz \citep{Garrigos23arxiv-HandbookConvergence}.

Using this machinery, we first derive the region of convergence for the stochastic gradient descent algorithm solving~\eqref{eq:objective}.
\begin{theorem}[Convergence Region of SGD] 
\label{thm:convergence-sgd}
Let $f_{i, I}$ be $L$-smooth and $\mu$-Polyak-Lojasiewicz. Then, the stochastic gradient descent algorithm (with update $\vw_{t+1} = \vw_t - \eta \nabla f_i(\vw_t)$) converges to the optimal value, namely $\E[\norm{f_I(\vw_t) - f^{\ast}_I}|\vw_0] < \epsilon$, provided all the model weights $\vw_t$ lie in the region $\setW_{\text{SGD}}$ given by 
	\begin{equation}
		\setW_{\text{SGD}} = \left\{ \vw \in \Real{d}~\Bigg|~\frac{1}{\Nout} \sum_{i \in \Nout}  \norm{ \vh_i(\vo_i, \vw) }^2  <  M \right\}, \nonumber
	\end{equation} 
	and
	$\eta < \frac{\mu}{L}\min\left\{ \frac{1}{L}, \frac{\epsilon}{3\FracOut M + 2 L \Delta_{f_I}}\right\}$, for some $M > 0$; where $\Delta_{f_I} = \frac{1}{n}( f^{\ast}_{I} - \min_{\vw} f_{i, I}(\vw))$.
\end{theorem}
We next analyze the region of convergence of \algoNameO. 
\begin{theorem}[Convergence Region of \algoNameO] 
\label{thm:convergence-adaptive}
Let $f_{i, I}$ be $L$-smooth and $\mu$-Polyak-Lojasiewicz. Furthermore, let $\nabla f_{i, I}(\vw)\tran \nabla f_I(\vw) \geq 0$ for all $i$. %
Then, \algoNameO (with update $\vw_{t+1} = \vw_t - \eta \sigma'_{c_t}(f_i(\vw_t))\nabla f_i(\vw_t)$ and $c_t$ chosen such that $\frac{1}{n}\sum_{i=1}^{n}\sigma'_{c_t}(f_i(\vw_t)) = \zeta$) converges to an $\epsilon$-neighborhood of the outlier-free optimal value $f^{\ast}_I$, namely $\E[\norm{f_I(\vw_t) - f^{\ast}_I}|\vw_0] < \epsilon$, provided all the model weights $\vw_t$ lie in the region $\setW_{\algoNameO}$ given by 
	\begin{equation}
		\setW_{\algoNameO} = \left\{ \vw \in \Real{d}~\Bigg|~
		\begin{array}{c} 
			 ~\exists c ~\text{s.t.}~~~~~~~~\frac{1}{\Nout} \sum_{i \in \Nout}  \sigma'_{c}(f_i(\vw))^2\norm{ \vh_i(\vo_i, \vw) }^2  < M, \\
			\vspace{-2mm}
			\frac{1}{n}\sum_{i=1}^{n}\sigma'_{c}(f_i(\vw)) = \zeta, ~\text{and}~ \min_i \sigma_c(f_i(\vw)) \geq \beta > 0
		\end{array}
		~\right\},
	\end{equation} 
	and $\eta < \frac{\mu \beta}{L}\min\left\{ \frac{1}{L}, \frac{\epsilon}{3\FracOut M + 2 L \Delta_{f_I} \zeta}\right\}$, for some $M > 0$; where $\Delta_{f_I} = \frac{1}{n}( f^{\ast}_{I} - \min_{\vw} f_{i, I}(\vw))$.
\end{theorem}
The set $\setW_{\algoNameO}$ has two more constraints $\frac{1}{n}\sum_{i=1}^{n}\sigma'_{c}(f_i(\vw)) = \zeta$ and $\min_i \sigma_c(f_i(\vw)) \geq \beta > 0$. The first comes from the step to update the parameter $c$ in the algorithm and is always satisfied. The second is a technical assumption required for the proof to hold. This will hold true for all continuously differentiable $\sigma$. Therefore, the key constraint that determines $\setW_{\algoNameO}$ is 
\begin{equation}
\label{eq:roi-adaptive}
	\frac{1}{\Nout} \sum_{i \in \Nout}  \sigma'_{c}(f_i(\vw))^2\norm{ \vh_i(\vo_i, \vw) }^2  < M.
\end{equation}
Comparing this to the constraint that defines $\setW_{\text{SGD}}$ we see a multiplicative factor of $\sigma'_{c}(f_i(\vw))^2$ appear before the summation. %
\begin{remark}[Increased Region of Convergence]
Firstly, note that both SGD and \algoNameO converge to a neighborhood of the outlier-free optima  when there are no outliers; this is true because of the $L$-smoothness and $\mu$-Polyak-Lojasiewicz assumptions \citep{Garrigos23arxiv-HandbookConvergence}. 
	The presence of outliers shrinks the region of convergence for both algorithms. However, the region of convergence for \algoNameO, \ie $\setW_{\algoNameO}$, is larger than $\setW_{\text{SGD}}$. This is because the constraint~\eqref{eq:roi-adaptive} is weaker than the one that defines $\setW_{\text{SGD}}$. Thus, the use of the robust loss kernel $\sigma$ (and the coefficient weighting with $\sigma'$) widens the region of convergence. 
\end{remark}

\begin{remark}[Convergence and the Fraction of Outliers $\FracOut$]
	The robust statistics literature has investigated the notion of breakdown point, which is a fraction of outlier samples that the estimator can handle, after which the estimator can produce arbitrarily bad estimates (\cite{Huber81}). A similar notion could be investigated for robust training algorithms.
	However, we have found it hard to obtain an explicit relation between convergence and the fraction of outliers $\FracOut$ in the training data. 
	Our result instead shows how the robust loss kernel $\sigma_c$ diminishes the impact of outliers in determining the region of convergence (see $\setW_{\algoNameO}$).  
\end{remark}

The proof of Theorem~\ref{thm:convergence-adaptive} relies on deriving an iterative relation between $\delta_{t+1}$ and $\delta_t$, where $\delta_t = \E[f_I(\vw_t) - f^{\ast}_I|\vw_0]$. 
Extending this to \algoNameT, when $T > 1$, poses a challenge. 
When $T > 1$ %
the coefficient weights $u_{i, t}$ at iteration $t$ are determined by $\vw_s$ at iteration $s$, for all $t \in [s, s+T]$ and $s \in \{0, T, 2T, \ldots \}$. 
We show the following result for \algoNameT:

\begin{theorem}[Convergence Region of \algoNameT] 
	\label{thm:convergence-gnc}
Let $f_{i, I}$ be $L$-smooth and $\mu$-Polyak-Lojasiewicz. Furthermore, let $\nabla f_{i, I}(\vw)\tran \nabla f_I(\vw) \geq 0$ and let $\setR(\vw)$ denote the region where all past $T$ iterates lie (\ie $\vw_{t'} \in \setR(\vw_t)$ for all $t' \in [t-T, t]$), given $\vw_t = \vw$.
Then, \algoNameT (with update $\vw_{t+1} = \vw_t - \eta \sigma'_{c_t}(f_i(\vw_s))\nabla f_i(\vw_t)$, for all $t \in [s, T + s]$ and $s \in \{0, T, 2T, \ldots\}$, and $c_t$ chosen such that $\frac{1}{n}\sum_{i=1}^{n}\sigma'_{c_t}(f_i(\vw_s)) = \zeta$ for all $s$) converges to an $\epsilon$-neighborhood of the outlier-free optimal value $f^{\ast}_I$, namely $\E[\norm{f_I(\vw_t) - f^{\ast}_I}|\vw_0] < \epsilon$, provided all the model weights $\vw_t$ lie in the region $\setW_{\algoNameT}$ given by 
	\begin{equation}
		\setW_{\algoNameT} = \left\{ \vw \in \Real{d}~\Bigg|~
		\begin{array}{c}
		 	~\max_{(\vw', c) \in \setH(\vw)} \frac{1}{\Nout} \sum_{i \in \Nout}  \sigma'_{c}(f_i(\vw'))^2\norm{ \vh_i(\vo_i, \vw) }^2  < M \\
			\vspace{-2mm} 
			 \text{and}~~ \min_i \sigma_c(f_i(\vw)) \geq \beta > 0
		\end{array}
		~\right\},
	\end{equation} 
	for some $M > 0$, where $\setH(\vw) = \left\{ (\vw', c)~\big|~ \vw' \in \setR(\vw)~\text{and}~c~\text{s.t.} ~\frac{1}{\Nout}\sum_{i \in \Nout}  \sigma'_{c}(f_i(\vw')) = \zeta \right\}$, provided $\eta < \frac{\mu \beta}{L}\min\left\{ \frac{1}{L}, \frac{\epsilon}{3\FracOut M + 2 L \Delta_{f_I} \zeta}\right\}$ with $\Delta_{f_I} = \frac{1}{n}( f^{\ast}_{I} - \min_{\vw} f_{i, I}(\vw))$.
\end{theorem}
The \algoNameT requires the quantity 
\begin{equation}
	 \frac{1}{\Nout} \sum_{i \in \Nout}  \sigma'_{c}(f_i(\vw'))^2\norm{ \vh_i(\vo_i, \vw) }^2,
\end{equation}
to remain bounded, where $\vw'$ are the model weights of any of the previous $T$ iterations. 
This is the same condition as was required for \algoNameO.
The space $\setH(\vw)$ is going to be larger for larger $T$, which makes sense, as belief about the outliers computed $T$ iterations earlier is likely to be stale now and impact convergence.

\begin{remark}[Convergence in Robust Estimation]
	A line of prior work has investigated convergence of iteratively re-weighted least square type algorithms. \cite{Aftab15wacv-ConvergenceIteratively} were the first to observe that it is the concavity property of $\rho(\sqrt{r})$ that ensures that the loss decreases for the iteratively re-weighted least squares (here $\rho$ is the robust loss as in Section~\ref{sec:url-robust-estimation}). They argued for concavity of $\rho(\sqrt{r})$ to be a necessary property for every robust loss design. This property translates to concavity of the robust loss kernel $\sigma$ and is satisfied by Definition~\ref{def:robust-loss-kernel}. Recent work \citep{Peng23cvpr-ConvergenceIRLS} derived two new graduated non-convexity algorithms for robust estimation, and for the first time, proved that they converge to the local optima of the robust M-estimation objective, albeit perturbed by $\epsilon$. These convergence results however did not investigate convergence of the iterates to the outlier-free optima $f^{\ast}_I$. 
	Moreover, these works focus on the robust estimation problems and, therefore, do not consider the deep learning setup where the training is inherently stochastic due to finite batch sizes. %
\end{remark}

\section{Experiments}
\label{sec:expt}

We experimentally demonstrate our theoretical results. We show that 
the \algoNameLongCaps
achieves lower variance in gradient computation and leads to better outlier mitigation. %
We observe that the algorithm is able to retain performance even when the percentage of outliers $\FracOut$ is large. 
We demonstrate this in three applications: linear regression, image classification, and  neural scene rendering~\citep{Mildenhall20arxiv-nerf,Mueller22acm-instantngp,Tancik23siggraph-nerfstudio}. The first two experiments primarily show the general applicability of our training algorithms and validate the theoretical results. The third experiment shows that the algorithm can be applied to mitigate pixel-level outliers in novel view synthesis problems based on neural radiance fields. 

We implement three variations of the \algoNameT (Algorithm~\ref{algo:athree}): 
(i) 
\emph{Adaptive TL}: 
\algoNameO with truncated loss kernel, 
(ii) 
\emph{Adaptive GM}: 
\algoNameO with Geman McClure loss kernel, 
(iii) 
\emph{Adaptive-T GM}:
\algoNameT with Geman McClure loss kernel. See Table~\ref{tab:robust-loss-kernels} for all the robust loss kernels. 

\subsection{Linear Regression}
\label{sec:expt-linear-regression}

We first consider the simple problem of linear regression.
    Given $n=1000$ measurement pairs $(\vxx_i, y_i) \in \Real{k} \times \Real{}$, we estimate a vector $\hat{\vw} \in \Real{k}$ 
    that minimizes a mean squared error (MSE) loss $f(\vw) = \frac{1}{n} \sum_{i=1}^{n} (y_i - \vw^T \vxx_i)^2$.
    We generate the measurement pairs $(\vxx_i, y_i)$ by first 
    sampling each coordinate of $\vxx_i$ uniformly randomly from $(0, 1]$ and $\vw^\star$ from $\calN(0, 1)$,
    and compute $y_i = \vw^\star \vxx_i + \epsilon_i + o_i$, where $\epsilon_i \sim \calN(0, 0.1)$ is a noise term
    and $o_i$ is the outlier term sampled from $\calN(0, 5)$, if $i \in \Nout$, and is otherwise set to zero.
    We vary $\FracOut$ (the fraction of outliers, \ie $\Nout/n$) from 0\% to 90\% with a 10\% increment. We average over five Monte Carlo trials for each $\FracOut$.
    For all methods, step size $\eta$ is set to $7\times10^{-4}$ and number of iterations is fixed at $10^{4}$. We use batch size of one in training.

Figure~\ref{fig:toy-training-loss} plots the test accuracy (\ie root mean squared error (RMSE)) as a function of fraction of outliers $\FracOut$ in the training data. We observe that even though MSE is noise-tolerant (see Section~\ref{sec:url-deep-learning}), the  SGD algorithm does not converge. This is because the outliers tend to induce high-variance during each descent iteration. Figure~\ref{fig:training-algo-variance} shows a training instance and how the variance affects convergence. The Adaptive GM and Adaptive TL reduce this variance and show better convergence. Gradient descent converges to the outlier-free optima correctly. This shows that the notion of noise-tolerance is useful when one has low variance in the estimation of the gradients.

 \begin{figure}
\centering
\begin{subfigure}{.42\textwidth}
  \centering
  \includegraphics[trim={0 0 0 0},clip,width=\linewidth]{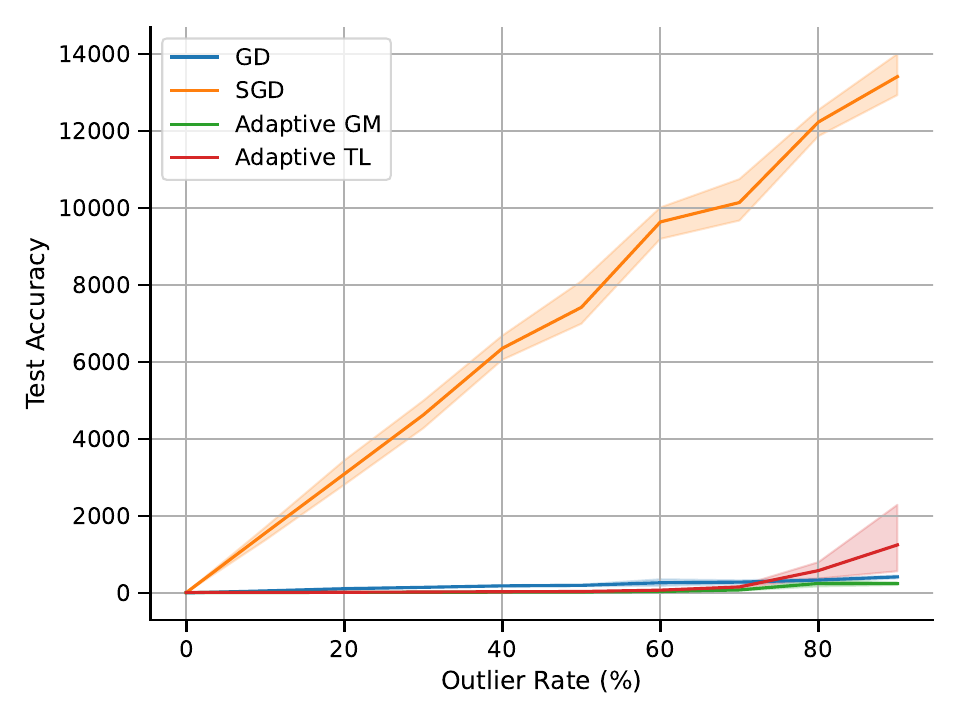}
  \vspace{-5mm}
  \caption{}
  \vspace{-5mm}
  \label{fig:toy-training-loss}
\end{subfigure}%
\hspace{5mm}
\begin{subfigure}{.43\textwidth}
  \centering
  \includegraphics[trim={10 10 20 10},clip,width=\linewidth]{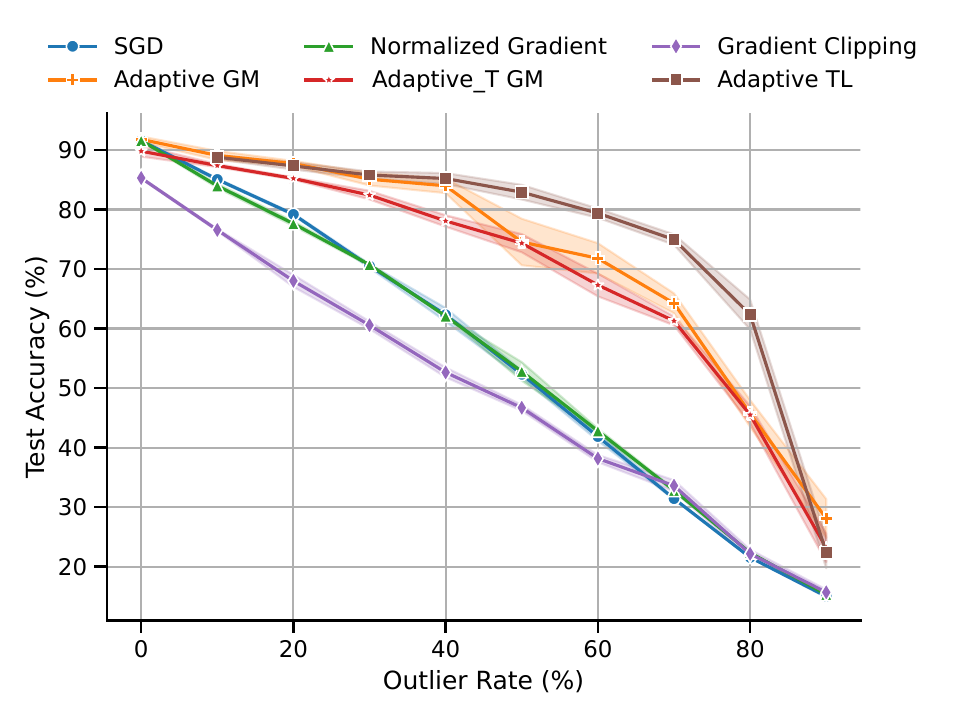}
  \caption{}
  \label{fig:cifar10-training-loss}
  \vspace{-2mm}
\end{subfigure}
\caption{(a) Test accuracy (\ie RMSE on test data) as a function of outlier fraction $\FracOut$ in the training data. The figure shows the gradient descent (GD) algorithm, stochastic gradient descent (SGD) algorithm, and two \algoNameLong{}s Adaptive GM and Adaptive TL. 
(b) Test classification accuracy as a function of outlier fraction $\FracOut$ in the training data. The figure shows SGD, Normalized Gradient Descent, Gradient Clipping, and the three \algoNameLong\!\!s Adaptive GM, Adaptive TL, and Adaptive-T GM.}%
\label{fig:training-loss}
\end{figure}

\subsection{Image Classification}
\label{sec:expt-classification}

We train a standard DLA-34 \citep{Yu18cvpr-DLA} network on the CIFAR10 datasets, with the standard train and test splits.
All methods are trained with a total of 500 epochs and the batch size of 128 and use cross-entropy loss.
To generate noisy labels, we adopt the standard symmetric noise model where sample labels are replaced following a uniform distribution of probability.
We vary the fraction of outliers $\FracOut$ in the training set from 0\% to 90\% with 10\% increment.
We implement SGD with momentum with fixed learning rate of 1e-3 and a momentum of 0.9, and use it as the gradient-based training algorithm in the implementation of the 
\algoNameLong.
We implement gradient clipping~\citep{Menon20iclr-clippinglabelnoise} and normalized gradient descent~\citep{Zhang20iclr-WhyGradient} for baseline comparisons. 
For all methods a weight decay of 5e-4 is applied during training.

Figure~\ref{fig:cifar10-training-loss} plots test accuracy as a function of the outlier ratio $\FracOut$. We observe that the Adaptive TL, Adaptive GM, and Adaptive-T GM show improved mitigation of outliers as opposed to simply training with the SGD algorithm. This validates our results in Section~\ref{sec:theory-roi} which argue that the 
\algoNameLongCaps has 
a larger region of convergence.

\subsection{Neural Radiance Field}
\label{sec:expt-nerf}
\begin{figure}
\centering
\begin{subfigure}{.4\textwidth}
  \centering
  \includegraphics[trim={38 18 25 15},clip,width=\linewidth]{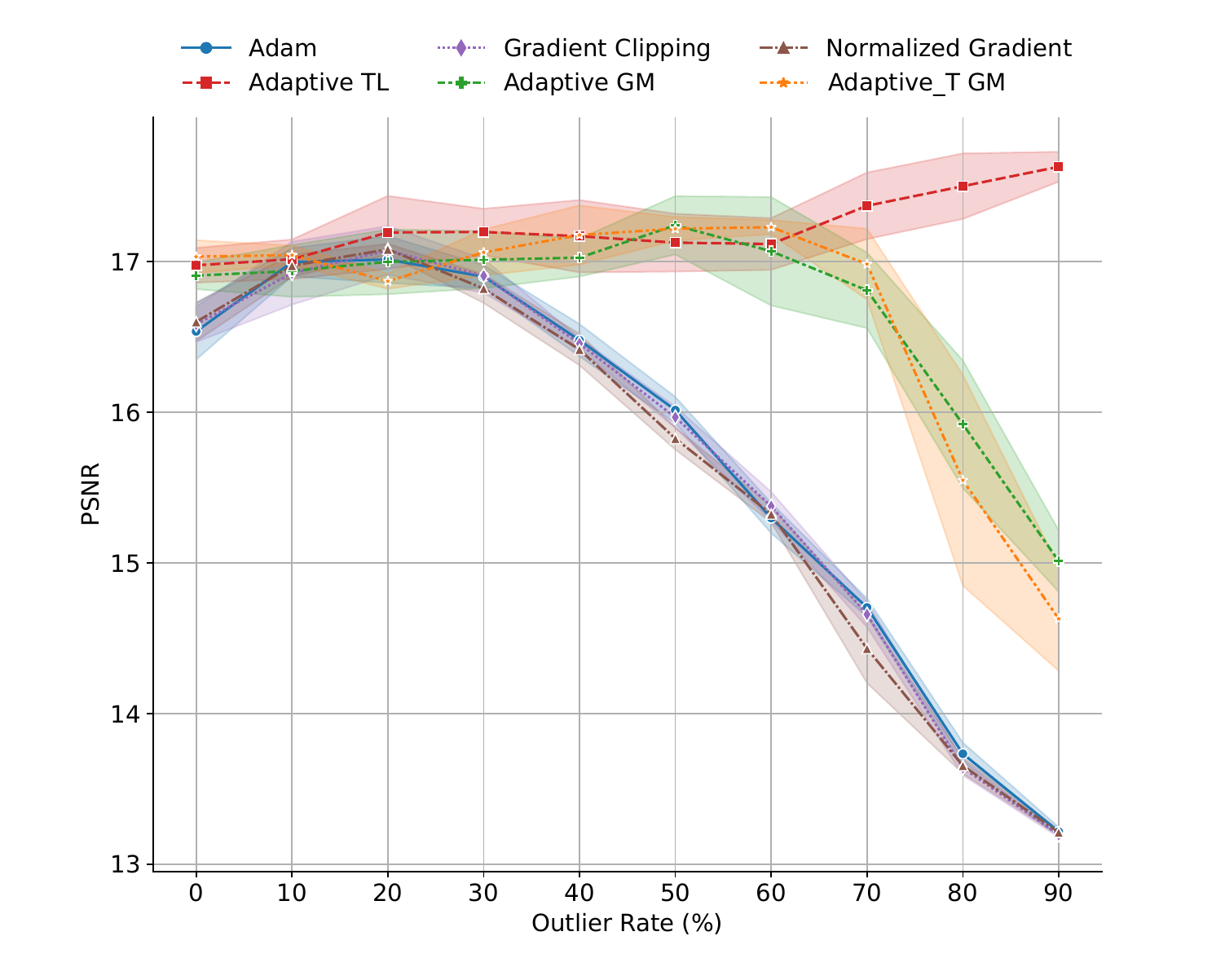}
  \caption{}
  \label{fig:nerf-psnr}
\end{subfigure}%
\begin{subfigure}{.38\textwidth}
  \centering
  \includegraphics[trim={38 18 80 15},clip,width=\linewidth]{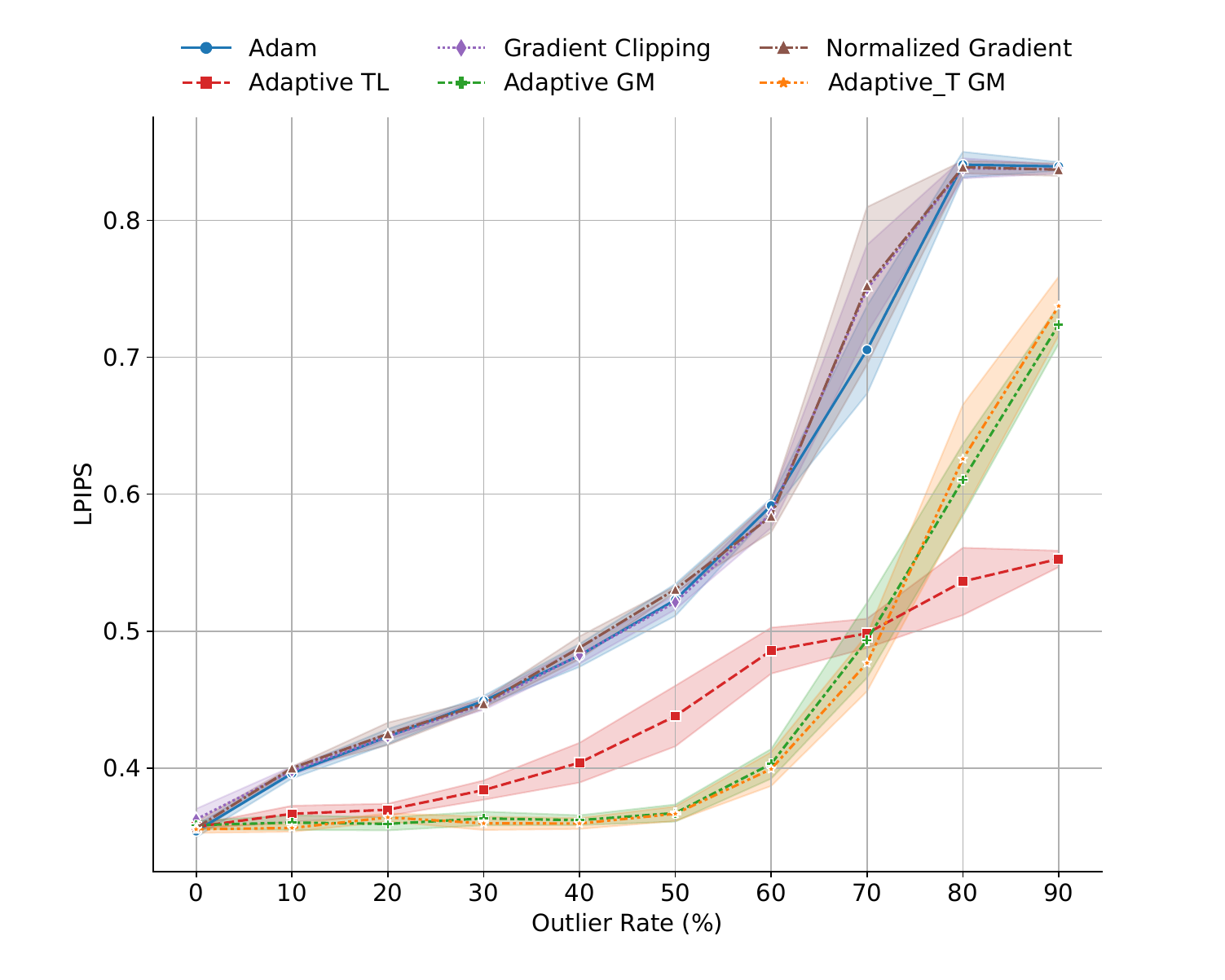}
  \caption{}
  \label{fig:nerf-lpips}
\end{subfigure}
\caption{Test accuracy (PSNR $\uparrow$ and LPIPS $\downarrow$) of the trained model as a function of \% outliers in the training data for various training algorithms: (i) Adam / SGD, the baseline approach proposed for training without outliers; (ii) Gradient Clipping, (iii) Normalized Gradient, (iv) Adaptive TL, (v) Adaptive GM, and (vi) Adaptive-T GM. 
\vspace{-7mm} 
}
\vspace{1mm}
\label{fig:nerf}
\end{figure}
\begin{wrapfigure}{r}{0.52\textwidth}
  \begin{center}
  \vspace{-8mm}
    \includegraphics[trim={7 10 30 10},clip,width=0.45\textwidth]{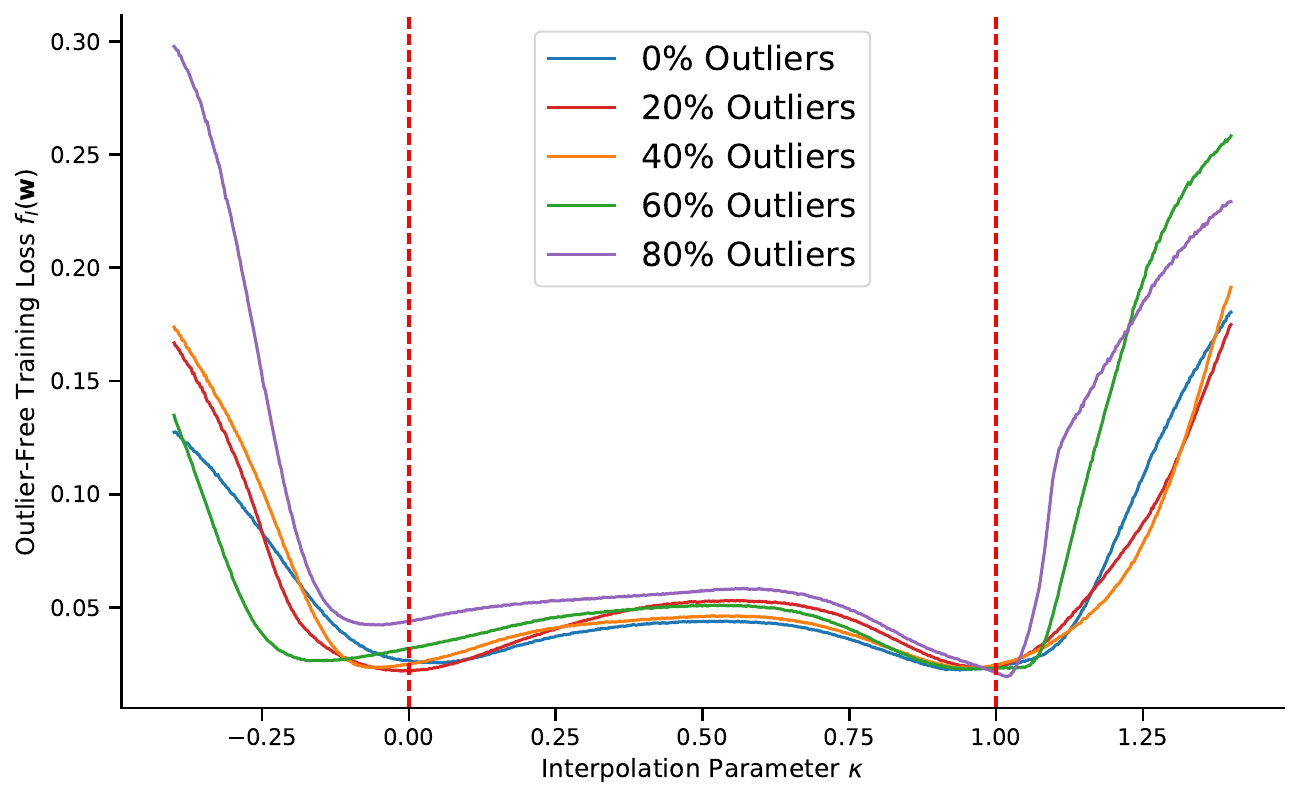}
  \end{center}
  \vspace{-1mm}
  \caption{Plot of the 1D training loss landscape as interpolated between the Adaptive TL model weight and the vanilla Adam model weights.}
  \vspace{-3mm}
  \label{fig:nerf-loss-landscape}
\end{wrapfigure}
We employ the open-source Nerfacto \citep{Tancik23siggraph-nerfstudio} model, 
a popular implicit scene reconstruction pipeline that combines Instant-NGP \citep{Mueller22acm-instantngp} with a 
camera-pose refinement stage.
We use the default model configuration parameters provided 
and an exponential decay scheduler with 2e5 steps with a final learning rate of 1e-4.
We simulate pixel-level noise by adding uniformly distributed noise to the camera ray originating from each pixel, which is selected with probability $\FracOut$.
We use the Adam optimizer with a learning rate of 1e-3.  
We implement gradient clipping~\citep{Menon20iclr-clippinglabelnoise} and normalized gradient descent~\citep{Zhang20iclr-WhyGradient} for baseline comparisons. 
We compare the methods with peak signal to noise ratio (PSNR) and learned perceptual image patch similarity (LPIPS) as in~\citep{Sabour23cvpr-robustnerf}. A higher PSNR value indicates better image quality, while a lower LPIPS score suggests greater perceptual similarity between the generated and the ground-truth images.

Figure~\ref{fig:nerf} plots two test accuracy metrics as a function of outlier rate $\FracOut$. 
We again observe 
that the Adaptive TL, Adaptive GM, and Adaptive-T GM show better robustness to outliers in the training data. Adaptive TL performs the best and shows good mitigation of outliers even when the training images have $90\%$ of the pixels degraded with outliers. To investigate a little more deeply
 the Adam convergence vis-a-vis our algorithms, we plot the 1D loss landscape in 
Figure~\ref{fig:nerf-loss-landscape}. 
The figure plots the 1D loss landscape as a function of an interpolation parameter $\kappa$ \citep{Li18nips-lossLandscapeNN}. %
The x-axis point $1$ is the optimal model weight the Adaptive TL training converges to, and the x-axis point $0$ is the optimal model weight the vanilla Adam converges to. We observe that the point to which Adam converges is a different local minima and is unstable in the presence of outliers (\ie we see the loss landscape wobbles as $\FracOut$ changes). On the other hand, the loss landscape near the model weight that Adaptive TL converges to remains the same, across the outlier rate $\FracOut$.  
Figure~\ref{fig:nerfacto-vis-results} shows views synthesized by two models: one trained with Adaptive TL and another trained with vanilla Adam, when we have an outlier rate of 80\% during training. We observe that the vanilla Adam is not able to recover any reasonable visual signal after training, while the Adaptive TL sees a visually good view synthesis.

\begin{figure}
\centering
\begin{subfigure}{.2\textwidth}
  \centering
  \includegraphics[width=\linewidth]{fig/nerf/tl/frame_00031.jpg}
  \label{fig:tl01}
\end{subfigure}
\begin{subfigure}{.2\textwidth}
  \centering
  \includegraphics[width=\linewidth]{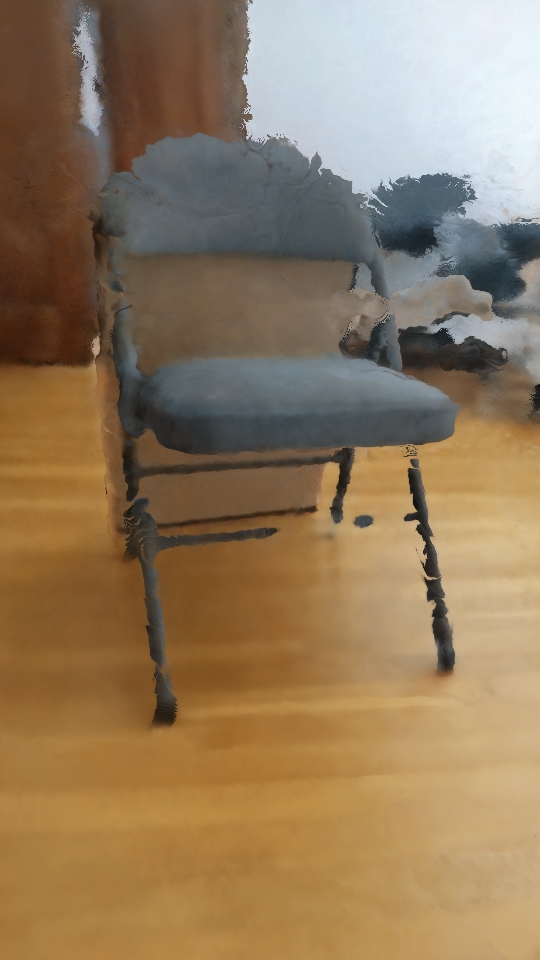}
  \label{fig:tl02}
\end{subfigure}
\begin{subfigure}{.2\textwidth}
  \centering
  \includegraphics[width=\linewidth]{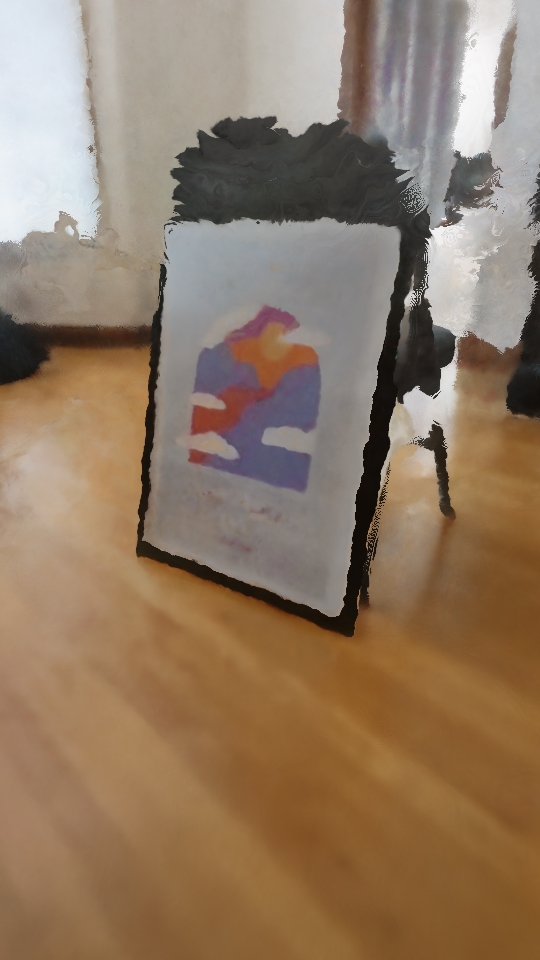}
  \label{fig:tl03}
\end{subfigure}
\begin{subfigure}{.2\textwidth}
  \centering
  \includegraphics[width=\linewidth]{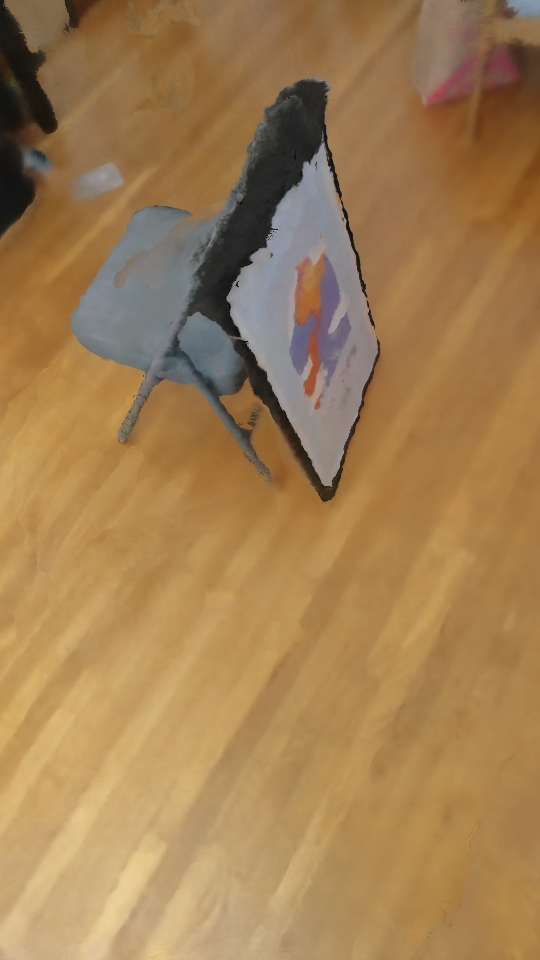}
  \label{fig:tl04}
\end{subfigure}
\begin{subfigure}{.2\textwidth}
  \centering
  \includegraphics[width=\linewidth]{fig/nerf/adam/frame_00031.jpg}
  \label{fig:tl01}
\end{subfigure}
\begin{subfigure}{.2\textwidth}
  \centering
  \includegraphics[width=\linewidth]{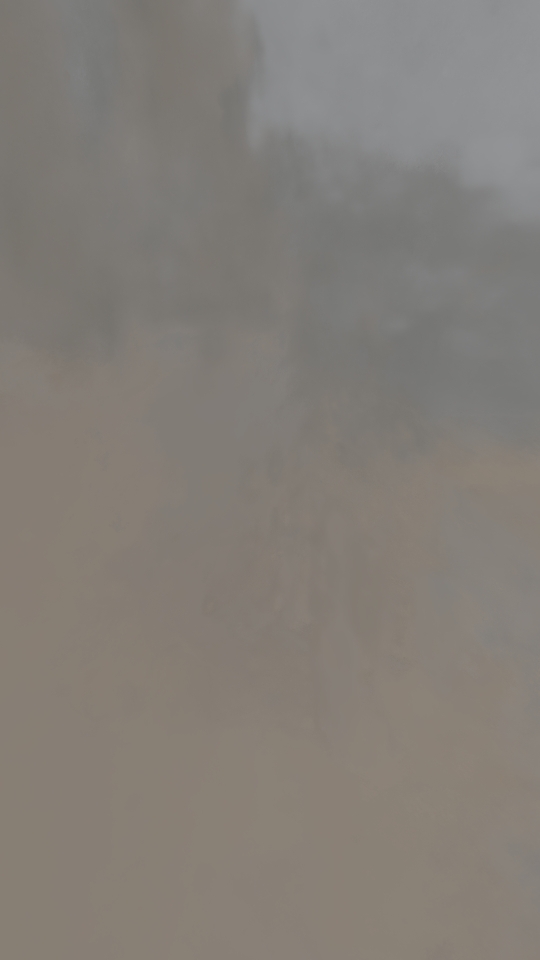}
  \label{fig:tl02}
\end{subfigure}
\begin{subfigure}{.2\textwidth}
  \centering
  \includegraphics[width=\linewidth]{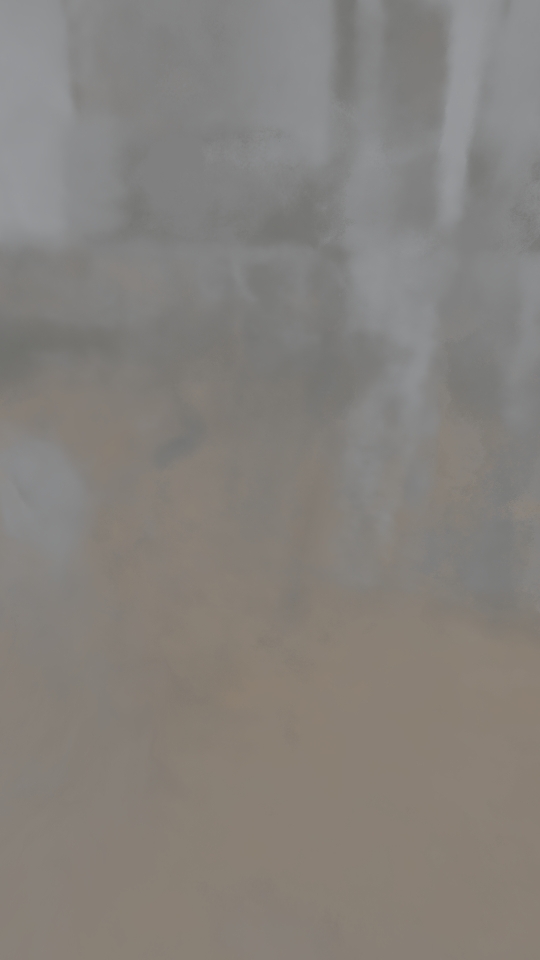}
  \label{fig:tl03}
\end{subfigure}
\begin{subfigure}{.2\textwidth}
  \centering
  \includegraphics[width=\linewidth]{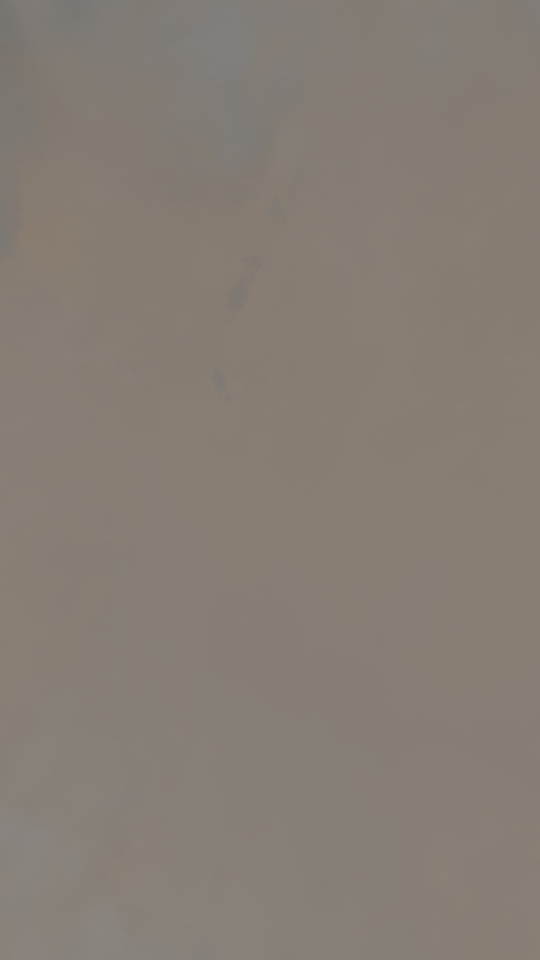}
  \label{fig:tl04}
\end{subfigure}
\begin{subfigure}{.2\textwidth}
  \centering
  \includegraphics[width=\linewidth]{fig/nerf/ground_truth/frame_00031.jpg}
  \label{fig:tl01}
\end{subfigure}
\begin{subfigure}{.2\textwidth}
  \centering
  \includegraphics[width=\linewidth]{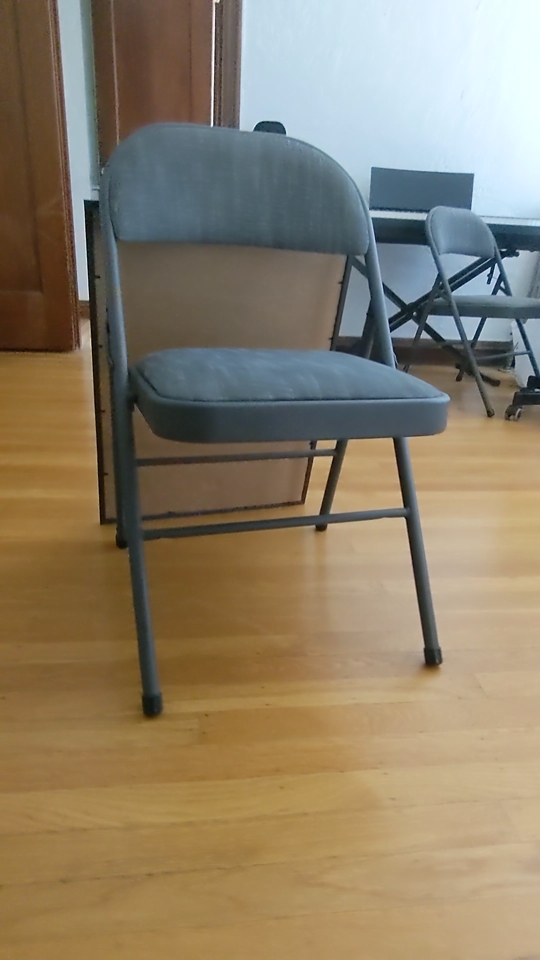}
  \label{fig:tl02}
\end{subfigure}
\begin{subfigure}{.2\textwidth}
  \centering
  \includegraphics[width=\linewidth]{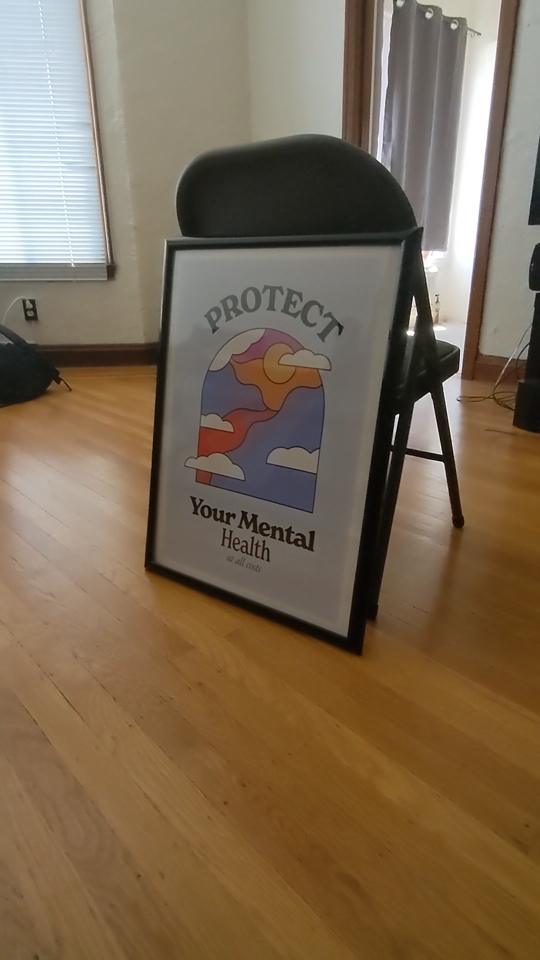}
  \label{fig:tl03}
\end{subfigure}
\begin{subfigure}{.2\textwidth}
  \centering
  \includegraphics[width=\linewidth]{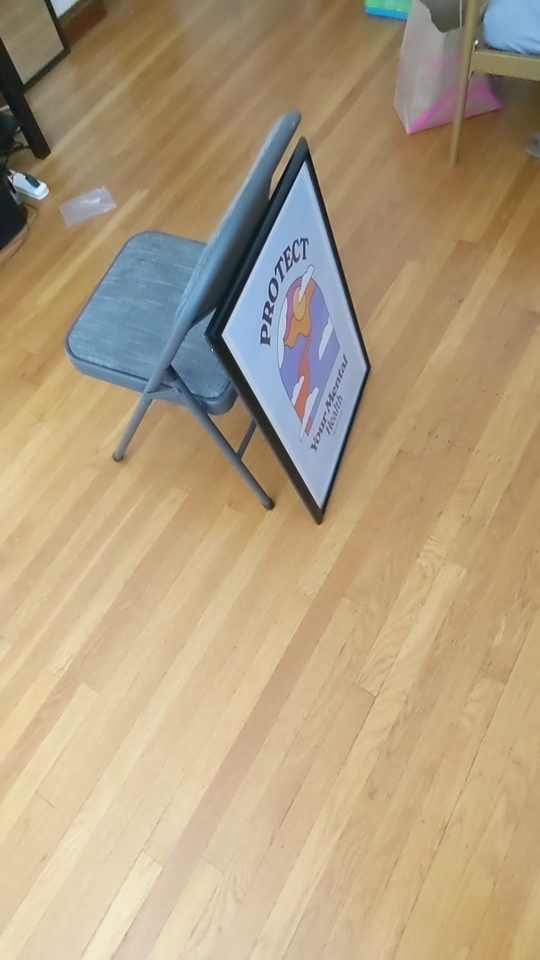}
  \label{fig:tl04}
\end{subfigure}
\caption{Nerfacto reconstruction results after $80\%$ of the training pixels have been perturbed by noise. The first row shows the result of training with our \algoNameLongCaps with Truncated Loss. The second row shows the result of running the original  Adam optimizer. The third row shows the ground truth images from the corresponding views.}
\label{fig:nerfacto-vis-results}
\end{figure}

\section{Related Work}
\label{sec:lit}

\subsection{Outlier Robust Training of Deep Learning Models}
\label{sec:lit-robust-dl}

Training image classification models in the presence of outliers in the training data has been well extensively investigated in the last decade. 
Existing methods include label correction methods, loss correction methods, refined training strategy, and robust loss function design. \cite{Algan21kbs-Imageclassification, Song23tnls-LearningNoisy} provide a detailed review on the topic of training multi-label classifiers in the presence of outliers in the training data.
 State-of-the-art approaches (\eg \cite{Li20iclr-dividemix}) use a combination of these approaches to attain best results. 
While most approaches remain specific to the task of image classification, some of them are generally applicable. Two such approaches include robust loss design and outlier-robust training algorithms.

Seminal works \citep{Ghosh15arxiv-MakingRisk, Ghosh17arxiv-RobustLoss} introduce the notion of \emph{noise-tolerant loss} (if $\vw^{\ast}_{\lambda}$ denotes the optimal model weights when minimizing a loss function $l$, then $l$ is said to be noise-tolerant to $\lambda$ fraction of outliers if $\vw^{\ast}_{\lambda} = \vw^{\ast}_{0}$). The paper goes on to prove that symmetric losses, such as a simple mean absolute error (MAE), are noise tolerant for multi-label classification provided $\lambda < 1 - 1/K$, where $K$ denotes the set of all label classes. The classical cross entropy (CE) loss is shown to be not noise tolerant.
Several works since then have investigated the design of robust and noise-tolerant losses. 
\cite{Zhang18nips-GeneralizedCross} %
propose the generalized cross entropy (GCE) loss that generalizes MAE and the CE loss, and is inspired by the negative Box-Cox transformation \citep{Box64jrss-AnalysisTransformations} and the generalized maximum likelihood framework \citep{Ferrari10astats-MaximumLqlikelihood}. %
\cite{Amid19nips-RobustBiTempered} replace the logarithms and exponentials in the cross-entropy loss with `tempered' versions %
\citep{Naudts02physica-Deformedexponentials}. The temperature parameters are tuned to achieve better outlier robustness. 
\cite{Wang19iccv-SymmetricCross} propose symmetric cross entropy loss, along the lines of symmetric KL divergence. 
\cite{Feng20ijcai-CanCross}  propose a loss that is a finite Taylor series expansion of the log likelihood loss.
\cite{Zhou23pami-AsymmetricLoss} propose an asymmetric loss and show how popular robust losses can be turned into noise-tolerant losses, under dominant clean label assumption (\ie the label noise is such that clean label remains dominant in the noise induced distribution).
\cite{Ma20icml-NormalizedLoss} show that any loss can be converted to noise-tolerant loss by applying a simple normalization. However, this changes the structure of the loss and can cause underfitting or divergence. They propose active-passive loss that combines two noise-tolerant loss functions that can boost each other. %
Curriculum and peer losses are proposes in \citep{Lyu20iclr-CURRICULUMLOSS} and \citep{Liu20icml-PeerLoss}, respectively. \cite{Xu19nips-L_DMINovel} propose determinant-based mutual information loss and show that it can successfully tackle instance-independent noise.

Outlier robustness is not only achieved by designing better robust losses, but also by developing better training strategies. 
\cite{Zhang17iclr-mixup} propose using convex combination of training samples to have the networks favor linear behaviors.
\cite{Elesedy23arxiv-uclip} maintain a buffer of clipped gradients and add them to the next iteration, show that their clipped updates are unbiased, and develop convergence guarantees under some assumptions. %
DivideMix~\citep{Li20iclr-dividemix} proposes a semi-supervised approach to refine noisy labels during training using a mixture model.
\cite{Menon20iclr-clippinglabelnoise} investigate the effect of gradient clipping on countering label noise in training classification networks.
\cite{Mai21icml-stability} develop quantitative results on the convergence of clipped stochastic gradient descent for non-smooth convex functions.
\cite{Ren18icml-LearningReweight} propose an iterative re-weighting scheme in training machine learning models in presence of outliers. It uses a small set of clean samples to evaluate and update the weights at each iteration.
Recent works have considered training neural radiance fields (NeRF) in the presence of distractors (\eg moving objects, lighting variations, shadows). RobustNeRF by \cite{Sabour23cvpr-robustnerf}, among other heuristics, uses median to trim outliers in loss computation during training.
\emph{While these training methods have shown promising results,
their convergence has not been analyzed.}

\subsection{Robust Estimation in Robot and Computer Vision}
\label{sec:lit-robust-cv}

Outlier-infested data is common in robot and computer vision, spanning estimation problems arising in robot localization and mapping, camera pose estimation, calibration, and 3D reconstruction.
For low-dimensional problems, RANSAC remains the go-to approach: RANSAC \citep{Fischler81} samples a small subset of measurements, and solves the problem using only those measurements (\ie using minimal solvers). Then, it identifies all the other measurements that are consistent with the solution, looking for large sets of measurements that ``agree'' with each other. RANSAC is fast for problems with a small minimal set and relatively low fractions of outliers, but is not guaranteed to converge to globally optimal solutions. %
More recently, certifiable outlier-robust methods~\citep{Yang20tro-teaser, Yang22pami-certifiablePerception} have shown how to 
frame several M-estimation problems in robotics and vision as a polynomial optimization problem, which can then be solved to certifiable optimality via standard semidefinite relaxations.  
While these methods yield certifiably optimal solutions, they tend to be computationally expensive. 

Graduated Non-Convexity (GNC)~\citep{Antonante21tro-outlierRobustEstimation, Yang20ral-GNC, Black96ijcv-unification, Peng23cvpr-ConvergenceIRLS} have emerged as a good balance between real-time computation and effective outlier mitigation for state estimation problems in robotics.  %
In it, the robust M-estimation problem is re-framed as an iterative re-weighted least squares. The weights indicates whether a measurement is an inlier or an outlier.  The duality result established in \cite{Black96ijcv-unification}  enables rephrasing M-estimation into a weighted least squares problem. 
This is popularly known as the Black-Rangarajan duality and is a common technique used to re-formulate and solve robust estimation problems. 

While GNC has shown promising results, a theoretical understanding about its convergence has been lacking in the robust estimation literature.
\cite{Aftab15wacv-ConvergenceIteratively} show that an iterative re-weighted least squares (IRLS) scheme, where the weights are updated according to the Black-Rangarajan duality, reduces the M-estimator loss, and can attain optimality, under very strict convexity conditions (\ie it requires the objective in~\eqref{eq:intro-m-est} to be convex). 
The IRLS scheme has been particularly studied in solving the Fermat-Weber problem. In it, the goal is to find a point that minimizes the \Lp distance from a given set of points. %
\cite{Brimberg93or-GlobalConvergence} investigate convergence of the IRLS procedure in this setting, whereas \cite{Aftab15pami-GeneralizedWeiszfeld} extend the IRLS scheme over Riemmanian manifolds (\eg \SOthree) and prove convergence for rotation averaging problems.
These works primarily tackle the case where the robust loss is assumed to be fixed.

Recent works have proposed parameterized (or \emph{adaptive}) robust losses to enable automatic tuning.  
\cite{MacTavish15crv-robustEstimation} was the earliest work in robot state estimation to show that adaptive robust losses improve outlier rejection.
\cite{Barron19cvpr-adaptRobustLoss} develops a general and adaptive robust loss function, that instantiates other well-known robust losses for different choices of the adaptive parameter. %
\cite{Chebrolu20ral-adaptiveCost} use this general robust loss and adapts its shape in training to better mitigate the outliers in robot state estimation problems.  
The GNC algorithm by \cite{Yang20ral-GNC, Antonante21tro-outlierRobustEstimation} has shown good practical performance, however, does not have any theoretical convergence guarantees. 
\cite{Peng23cvpr-ConvergenceIRLS} propose new GNC algorithms, by defining two new parameterized versions of the robust losses for \Lp and the truncated least squares loss. Unlike in \citep{Yang20ral-GNC}, they prove that their GNC algorithm converges to stationary points of the M-estimator, albeit perturbed by $\epsilon$. 
\cite{Shen19icml-LearningBad} propose iteratively training with a pre-defined fraction of `good' samples (\ie samples with the lowest loss). This algorithm can be thought of as using an adaptive truncated robust loss in each iteration, where the truncation threshold is adapted at each iteration. The paper also derives convergence to error bounds for a generalized linear model. 
\emph{While there is interest in developing better GNC algorithms with convergence properties, these results do not directly extend to the context of training deep learning models, where solvers use finite batch sizes. }

\subsection{Convergence Analysis of Training Algorithms in the Presence of Outliers}
\label{sec:lit-sgd-convergence}
Classical machine learning problems (\eg linear regression, principle component analysis, matrix decomposition)  have received significant attention, and many algorithms have been proposed to cope with outliers. 
Training deep learning models in the presence of outliers, however, remains challenging for at least two reasons. 
First, deep learning models are trained using batches of data, and therefore, any algorithm only has access to an estimate of the true gradient. A biased or an outlier gradient can significantly affect convergence. 
Second, the training loss can be non-convex and is hard to analyze without making certain assumptions.

Stochastic gradient descent is the most popular approach for training deep learning models. Several works have investigated its convergence behavior, brought to forth its limitations (\eg noise-variance issue, biased gradient estimates, convergence to non-flat local optima), 
and proposed variants to overcome them (see \cite{Demidovich23nips-GuideZoo, Zhang20iclr-WhyGradient, Zhang20nips-Whyare, Reisizadeh23arxiv-VariancereducedClipping, Koloskova23icml-Revisitinggradient, Gower20ieee-VarianceReducedMethods, Foret20iclr-SharpnessawareMinimization}). %
\cite{Garrigos23arxiv-HandbookConvergence} provide a comprehensive review on analysis techniques for proving convergence of the SGD algorithms, under different assumptions on the training loss such as $L$-smoothness, strong convexity, and $\mu$-Polyak-Lojasiewicz.
While analyzing SGD and its variants has been easier, convergence of the popular Adam optimizer remains elusive~\citep{Dereich24arxiv-Convergencerates}.

Very few works have considered the effect of outliers on the convergence of training algorithms, including SGD. 
\cite{Menon20iclr-clippinglabelnoise} were the first to point out that gradient clipping (albeit with small modifications) can be robust to outliers. They analyzed the special case of linear classification, with training batch size of one, and showed its equivalence to minimizing a Huberized and partially Huberized losses. They showed that their proposed gradient clipping algorithm provably exhibits a constant excess risk under symmetric label noise, in binary classification. 
\cite{Merad24tmlr-RobustStochastic} propose gradient quantile clipping, where the gradient clipping threshold is chosen to be the $p$th quantile of all the estimated gradient norms. The paper goes on to derive convergence property of the iterates, under $L$-smoothness and strong convexity assumptions. 
\cite{Chhabra24arxiv-OutlierGradient} draw a connection between identifying detrimental training sample (\ie a training sample that can unduly affect the model) and outlier gradient detection.
\cite{Hu24pmlr-OutlierRobust} formulate an adversarial training process, where for each given input-output sample, one estimates a worst-case input for each annotated output, and trains using the worst-case input-output pairs. The paper analyzes its $\calH$-consistency, generalizability, and convergence for the special case of binary classification.
\cite{Shen19icml-LearningBad, Shah20pmlr-ChoosingSample} propose to iteratively train the model with a subset of samples that have the lowest loss. It shows convergence results under strong convexity and bounded variance of gradient estimates used in the stochastic gradient descent. 
\cite{Prasad20jrss-robustEstimation} propose to robustly estimate the gradients, and shows convergence under two outlier models on gradients, namely, Huber contamination and heavy-tail distribution. 
The analysis in all these papers is stochastic in nature, \ie they assume an outlier distribution. 
\emph{On the contrary, our work studies the convergence of training algorithms in the presence of arbitrary outliers, without any distributional assumption.} 

We conclude by observing that robust estimation has been also the subject of intense study in the applied mathematics and statistics community. The corresponding papers have focused
on clustering and moment estimation~\citep{Lai16focs-momentEstimation,Diakonikolas16focs-robustMomentEstimation,Diakonikolas19siam-robustMomentEstimation,Charikar17stoc-robustEstimationTheory,Kothari17arxiv-clustering,Kothari18stoc-robustMomentEstimation,Diakonikolas18soda-robustLearningGaussians},
 subspace learning for classification in the presence of malicious noise~\citep{Klivans09alp-subspaceLearning,Diakonikolas18stoc-subspaceLearning,Awasthi17acm-robustLinearSeparators}, and {robust linear regression}~\citep{Klivans18arxiv-robustRegression, %
Diakonikolas19icml-robustRegression, %
Prasad20jrss-robustEstimation, %
Diakonikolas19soda-robustRegression, %
Bhatia17neurips-robustRegression, %
Karmalkar18arxiv-robustL1Regression,Karmalkar19neurips-ListDecodableRegression,Raghavendra20soda-ListDecodableRegressions}.  %
The literature includes approaches based on 
iterative outlier filtering~\citep{Diakonikolas19icml-robustRegression,Diakonikolas19soda-robustRegression}, 
robust gradient estimation~\citep{Prasad20jrss-robustEstimation}, 
hard thresholding~\citep{Bhatia17neurips-robustRegression,Bhatia15neurips-hardThresholding,Chen13icml-robustSparseRegression}, 
$\ell_1$-regression~\citep{Nguyen13tit-l1robustEstimation,Karmalkar18arxiv-robustL1Regression,Wright10tit-l1robustEstimation},
and moment/sum-of-squares relaxations~\citep{Klivans18arxiv-robustRegression,Karmalkar19neurips-ListDecodableRegression}. We refer the reader to~\citep{Carlone23fnt-estimationContracts} for a broader discussion.
\emph{We remark that our algorithm and analysis apply to generic deep learning problems, going beyond linear regression and specific instances of the learning problem}.  %

\newcommand{\omitit}[1]{}
\section{Conclusion}
\label{sec:conclude}
\omitit{
A simple modification of the Black-Rangarajan duality brings out a definition of a robust loss kernel $\sigma$, that unifies the robust losses in (a) robust estimation literature in robot and computer vision, and in (b) training deep learning models in presence of outliers. 
The unified robust loss kernel now creates an opportunity to cross-pollinate, \ie test robust kernels developed in the deep learning literature in robust estimation problems, and vice versa.

The modified duality shows that the robust M-estimation problem can be written as a weighted version of the original problem. This allows us to keep the problem structure in-tact: \eg (i) the dual of a robust (\ie M-estimator) non-linear least squares is a weighted non-linear least square, and (ii) the dual of a robust cross-entropy minimization is a weighted cross-entropy minimization). 
The modified duality helps us derive graduated non-convexity and adaptive training algorithms. We also develop a parameter update rule that obviates the need to do any parameter hand-tuning during training. We show that the resulting algorithms are generalized versions of previously conceived algorithms  (\eg \cite{Shen19icml-LearningBad, Shah20pmlr-ChoosingSample} and training on conformal prediction sets \citep{Shafer08jmlr-TutorialConformal}). 
We develop analysis to show convergence of these two class of algorithms. Our analysis sets the goal to reach the optima of the outlier-free objective, \ie the ground-truth. We show that the use of the robust loss kernel $\sigma$ increases the region of convergence for these two algorithms, in comparison to training with stochastic gradient descent. 
Unlike prior work, our analysis assume arbitrary outliers, with no distributional assumption. 
}

We present a simple modification of the Black-Rangarajan duality that leads to a definition of a robust loss kernel $\sigma$, which unifies the robust losses in (a) robust estimation in robotics and computer vision, and in (b) training deep learning models in the presence of outliers. 
The unified robust loss kernel $\sigma$ creates an opportunity to cross-pollinate, \ie test robust kernels developed in the deep learning literature in robust estimation problems, and vice versa. 
The modified Black-Rangarajan duality can now be applied to any machine learning problem, and not just those that adhere to a least squares loss. 
We also propose an \algoNameLong, which adds to the 
list of practitioners' tools to robustly train machine learning models. %
Moreover, we analyze convergence properties of the proposed algorithm.
The analysis techniques we use open the doors to further studies of convergence of training algorithms, under arbitrary outliers assumptions. While we present a general result, specific problem structure may be exploited, in the future, to understand the impact of robust loss kernels on convergence. %

\bibliographystyle{tmlr}

\appendix

\section{Proof of Corollary~\ref{cor:br-linear}}
\label{pf:cor:br-linear}

\myParagraph{Derivation from First Principles}
We show that the following two optimization problems are equivalent:  
\begin{equation}
	\label{eq:equivalence}
	\underset{r}{\text{Minimize}}~\sigma(r)~~\equiv~~\underset{r, u}{\text{Minimize}}~~ u \cdot r + \Phi(r),
\end{equation}
if $\Phi(u) = \sigma( (\sigma')^{-1}(u) ) - u (\sigma')^{-1}(u)$. This establishes the core of our modified Black-Rangarajan duality. Applying this to a sum of losses directly yields Corollary~\ref{cor:br-linear}.

Let $u(r)$ be the $u$ that minimizes $u \cdot r + \Phi(r)$. The first-order optimality condition suggests that $u(r)$ must satisfy
\begin{equation}
	\label{eq:01}
	r + \Phi'(u(r)) = 0.
\end{equation}
The equivalence~\eqref{eq:equivalence} will hold if $\sigma(r) = r\cdot u(r) + \Phi(u(r))$. Taking derivative with respect to $r$ on both sides of this equation yields
\begin{equation}
	\label{eq:02}
	\sigma'(r) = u(r) + r u'(r) + \Phi'(u(r)) u'(r) = u(r),
\end{equation}
where the last equality followed by using~\eqref{eq:01}. Using~\eqref{eq:02} in $\sigma(r) = r\cdot u(r) + \Phi(u(r))$ we obtain
\begin{equation}
	\label{eq:03}
	\sigma(r) = r\cdot \sigma'(r) + \Phi(\sigma'(r)).
\end{equation}
Now, note that $\sigma$ is a robust loss kernel (Definition~\ref{def:robust-loss-kernel}), and therefore satisfies $\sigma'(r) \in [0, 1]$ and $\sigma''(r) < 0$, \ie $\sigma'$ is strictly monotonic and has an inverse. Therefore, let $u = \sigma'(r) \in [0, 1]$ and $r = (\sigma')^{-1}(u)$. Substituting $r = (\sigma')^{-1}(u)$ in~\eqref{eq:03} yields
\begin{equation}
	\Phi(u) = \sigma( (\sigma')^{-1}(u) ) - u (\sigma')^{-1}(u), 
\end{equation}
for $u \in [0, 1]$. This proves the result.

\section{On Robust Losses for Multi-Label Classification}
\label{app:classification-robust-losses}
The identity $l(\vp, y) = \rho(- \log \vp[y])$ can be obtained by first substituting $\vp[y] = e^{-r}$ to obtain $\rho(r) = l(\vp, y)$. For example, for the generalized cross-entropy $l(\vp, y) = \frac{1}{q}(1 - \vp[y]^q)$ we have 
\begin{equation}
	\rho(r) = \frac{1}{q}\left( 1 - e^{-q r}\right).
\end{equation}
It then trivially follows that $l(\vp, y) = \rho(- \log \vp[y])$. 
In applying this to symmetric cross-entropy and reverse cross entropy we use the fact that $\sum_{k\neq y} \vp[k] = 1 - \vp[y]$.

\section{Robust Loss Kernels and Robust Losses}
\label{app:robust-losses-new}

The first six robust loss kernels are given by $\sigma_c(r) = \text{cost.} \cdot \rho_{\sqrt{c}}(\sqrt{r})$, where $\rho_c(r)$ are the robust losses in Section~\ref{sec:lit-robust-cv}; we have explicitly added the parameter $c$ in the notation $\rho_c(r)$. The constant multiple ensures that the kernel is scaled appropriately to satisfy Definition~\ref{def:robust-loss-kernel}. 

The remaining robust loss kernels are obtained by substituting $r = -\log(\vp[y])$ in the losses given in Section~\ref{sec:lit-robust-dl}. This is because we want the losses (in Section~\ref{sec:lit-robust-dl}) to be robust loss kernel of the cross-entropy loss. It can be analytically verified that all the kernels in Table~\ref{tab:robust-loss-kernels} satisfy Definition~\ref{def:robust-loss-kernel}.

\section{Proof of Lemma~\ref{lem:analytical-coefficient-update}}
\label{pf:lem:analytical-coefficient-update}

	Setting the derivative of the objective to zero, we obtain
	\begin{equation}
	\label{eq:t0}
		f_i(\vw) = -\Phi^{'}_{\sigma_c}(u^{\ast}).
	\end{equation}
	We know from the modified Black-Rangarajan duality (Corollary~\ref{cor:br-linear}) that $\Phi_{\sigma_c}(u) = -u (\sigma'_c)\inv(u) + \sigma_c( (\sigma^{'}_c)\inv(u))$. Taking its derivative we obtain
	\begin{equation}
		\Phi_{\sigma_c}'(u) = - (\sigma^{'}_c)\inv(u).
	\end{equation}
	Substituting this back in~\eqref{eq:t0} and applying $\sigma'_c$ on both sides, we obtain the result. %

\section{Verifying Assumption~\ref{as:outlier-gradient}}
\label{app:as:outlier-gradient}
We first verify that the outlier gradient assumption (Assumption~\ref{as:outlier-gradient}) holds for two broad class of problems, namely, non-linear regression and multi-label classification.
\begin{example}[Non-Linear Regression]
\label{ex:non-linear-regression}
	Consider a model $\vg$ that predicts the output $\vy = \vg(\vw, \vxx)$ given the input $\vxx$ and model weights $\vw$. The $i$-th measurement loss is the L2 norm given by $f_i(\vw) = \norm{\vy_i - \vg(\vw, \vxx_i)}^2$. 
	The model is trained on annotated data, which suffers from incorrect output annotations: $\vy_i = \vy^{\ast}_i$ for $i \in \Nin$, but $\vy_i = \vy^{\ast}_i + \vo_i$ for $i \in \Nout$; here $\vy^{\ast}_i$ denotes the correct annotation. The loss for the outlier-infested measurement becomes $f_i(\vw) = f_{i, I}(\vw) + \norm{\vo_i}^2 + 2 \vo_i\tran(\vy_i^{\ast} - \vg(\vw, \vxx_i))$, where $f_{i, I}(\vw) = \norm{\vy_{i}^{\ast} - \vg(\vw, \vxx_i) }^2$. 
	The gradient $\nabla f_{i}(\vw)$ of the outlier-infested objective is then given by 
	\begin{equation}
		\nabla f_{i}(\vw) = \nabla f_{i, I}(\vw) + \vh_i(\vo_i, \vw),
	\end{equation}
	with $\vh_i(\vo_i, \vw) = \nabla_\vw g(\vw, \vxx_i) \vo_i$. This satisfies Assumption~\ref{as:outlier-gradient}.
\end{example}
\begin{example}[Multi-Label Classification]
\label{ex:classification}
	Consider a model $\vp(\vw, \vxx)$ that predicts the probability that the input $\vxx$ belongs to which class, \ie  $\vp(\vw, \vxx)[y]$ denotes the predicted probability that the input $\vxx$ is of class $y \in [K]$. The model is trained on annotated data $\{ (\vxx, y_i) \}_{i \in [n]}$. The loss component is given by
	\begin{equation}
		f_i(\vw) = - \log \left( \vp(\vw, \vxx_i)[y_i] \right).
	\end{equation}
	Annotations are not perfect: $y_i = y^{\ast}_i$ for $i \in \Nin$, but this is not the case for outlier measurements. For outlier measurements, the loss can be re-written as 
	\begin{align}
		f_i(\vw) 
		&= - \log \left( \vp(\vw, \vxx_i)[y_i] \right), \\
		&=  - \log \left( \vp(\vw, \vxx_i)[y^{\ast}_i] \right)  - \log \left( \vp(\vw, \vxx_i)[y_i] / \vp(\vw, \vxx_i)[y^{\ast}_i] \right), \\
		&= f_{i, I}(\vw)  - \log \left( \vp(\vw, \vxx_i)[y_i] / \vp(\vw, \vxx_i)[y^{\ast}_i] \right).
	\end{align}
	Therefore, $\nabla f_i(\vw) = \nabla f_{i, I}(\vw) + \vh_i(\vo_i, \vw)$, where 
	\begin{equation}
		\vh_i(\vo_i, \vw) = \nabla_\vw \left[ - \log \left( \vp(\vw, \vxx_i)[y_i] / \vp(\vw, \vxx_i)[y^{\ast}_i] \right) \right],
	\end{equation}
	which satisfies Assumption~\ref{as:outlier-gradient}.
\end{example}
\section{Proof of Lemma~\ref{lem:training-algo-variance}}
\label{pf:lem:training-algo-variance}

For a batch size of one,
the gradients are given by
$\vg_t = \eta \nabla f_i(\vw_t)$ and $\vg_t = \eta \sigma'_c(\vw_t) \nabla f_i(\vw_t)$
for SGD and \algoNameO, respectively.
where $i$ is a uniformly distributed random variable over the set $[n]$, \ie $i \sim \calU([n])$. For SGD, note that the mean $\bar{\vg}_t = \E_i[\eta \nabla f_i(\vw_t)] = \eta \nabla f_I(\vw_t)$ because of the zero-mean assumption, \ie $\E_i[\vh_i(\vo_i, \vw)] = 0$. The variance is, therefore, given by %
\begin{align}
	\E_i[ \norm{\vg_t - \eta \nabla f_I(\vw_t)}^2] &= \frac{1}{n}\sum_{i =1}^{n} \norm{\eta \nabla f_{i, I}(\vw_t) + \eta \vh_i(\vo_i, \vw_t) - \eta \nabla f_I(\vw_t)}^2\\ 
	&= \eta^2 \E_i[ \norm{\nabla f_{i, I}(\vw_t)}^2 ] + \eta^2 \FracOut \frac{1}{\Nout}\sum_{i=1}^{\Nout} \norm{\vh_i(\vo_i, \vw_t)}^2 - \eta^2\norm{\nabla f_I(\vw_t)}^2 \nonumber\\
	&\hspace{-10mm}-2\eta^2 \FracOut \left(\frac{1}{\Nout}\sum_{i=1}^{\Nout}\vh_i(\vo_i, \vw_t)\right)\tran\nabla f_I(\vw_t) + 2\eta^2 \FracOut\frac{1}{\Nout}\sum_{i=1}^{\Nout} \nabla f_{i, I}(\vw_t)\tran \vh_i(\vo_i, \vw_t). \nonumber
\end{align} 
Using the facts: (i) $\vh_i$s are zero mean, (ii) $ \E_i[ \norm{\nabla f_{i, I}(\vw_t)}^2 ] \geq \norm{\nabla f_I(\vw_t)}^2$ due to Jensen's inequality, (iii) Cauchy–Schwarz inequality along with Assumption~\ref{as:low-signal-to-outlier}, we obtain
\begin{align}
	\E_i[ \norm{\vg_t - \eta \nabla f_I(\vw_t)}^2] &\leq 3 \eta^2 \FracOut \frac{1}{\Nout}\sum_{i=1}^{\Nout} \norm{\vh_i(\vo_i, \vw_t)}^2.
\end{align} 
Following the same line of argument one can derive the result for the \algoNameO. %

\section{Preliminary Lemmas}

We state and prove some results needed to establish the key result in the paper. We consider $i$ to be a uniformly distributed random variable over the set $[n] = \{1, 2, \ldots n\}$. The notation $\E_i[\cdot]$ refers to expectation with respect to $i$ and evaluates to a simple average: $\E_i[g(i)] = \frac{1}{n}\sum_{i=1}^{n}g(i)$.
\begin{lemma}
\label{lem:gradient}
$\E_{i}\left[\nabla f_i(\vw) \right] = \nabla f_{I}(\vw) + \FracOut \frac{1}{\Nout} \sum_{i \in \Nout} \vh_i(\vo_i, \vw)$.
\end{lemma}
\begin{proof}
	We know that $\nabla f_i(\vw) = \nabla f_{i, I}(\vw) + \vh_i(\vo_i, \vw)$ for all outlier measurements. For inlier measurements, $\nabla f_i(\vw) = \nabla f_{i, I}(\vw)$ as $f_i(\vw) = f_{i, I}(\vw)$. This implies
	\begin{align}
		\E_{i}\left[\nabla f_i(\vw) \right] 
		&= \frac{1}{n}\sum_{i=1}^{n} \nabla f_{i}(\vw), \\
		&= \frac{1}{n}\sum_{i \in \Nin} \nabla f_{i, I}(\vw) + \frac{1}{n}\sum_{i \in \Nout} \nabla f_{i, I}(\vw)  + \vh_i(\vo_i, \vw), \\
		&= \frac{1}{n}\sum_{i =1}^{n} \nabla f_{i, I}(\vw) + \FracOut \frac{1}{\Nout} \sum_{i \in \Nout} \vh_i(\vo_i, \vw), 
	\end{align}
	which proves the result. 
\end{proof}

\begin{lemma}
\label{lem:gradient-squares}
Let the low signal-to-outlier ratio assumption (Assumption~\ref{as:low-signal-to-outlier}) hold.
The second-order moments $\E_{i}\left[ \norm{\nabla f_i(\vw)}^2 \right]$ and $\E_{i}\left[\nabla f_i(\vw)\tran\nabla f_I(\vw) \right]$ are bounded by: 
\begin{equation}
\label{eq:lem:gradient-squares:eq01}
	\E_{i}\left[ \norm{\nabla f_i(\vw)}^2 \right] \leq 2 L (f_I(\vw) - f^{\ast}_I) + 2L \Delta_{f_I} + 3 \FracOut \frac{1}{\Nout} \sum_{i \in \Nout} \norm{\vh_i(\vo_i, \vw)}^2,
\end{equation}
and
\begin{equation}
\label{eq:lem:gradient-squares:eq02}
	\E_{i}\left[\nabla f_i(\vw)\tran\nabla f_I(\vw) \right] \geq 2\mu (f_I(\vw) - f^{\ast}),
\end{equation}
where $\Delta_{f_I} = \frac{1}{n}\sum_{i=1}^{n} (f^{\ast}_{I} - f^{\ast}_{i, I})$ and $f^{\ast}_{i, I} = \min_{\vw} f_{i, I}(\vw)$.
\end{lemma}
\begin{proof}
	Expand $\E_i[\norm{\nabla f_i(\vw)}^2]$ as
	\begin{align}
		\E_i[\norm{\nabla f_i(\vw)}^2] 
		&= \frac{1}{n}\sum_{i=1}^{n} \norm{\nabla f_i(\vw)}^2, \\
		&=\frac{1}{n} \sum_{i \in \Nin} \norm{\nabla f_{i, I}(\vw)}^2 + \frac{1}{n} \sum_{i \in \Nout} \norm{\nabla f_{i, I}(\vw) + \vh_i(\vo_i, \vw)}^2, \\
		&=\frac{1}{n} \sum_{i \in \Nin} \norm{\nabla f_{i, I}(\vw)}^2 \\
		&~~~~~~~~~~~~~~~~~~~~~~~~~~~~+ \frac{1}{n} \sum_{i \in \Nout} \norm{\nabla f_{i, I}(\vw)}^2 + \norm{\vh_i(\vo_i, \vw)}^2 + 2\vh_i(\vo_i, \vw)\tran \nabla f_{i, I}(\vw), \nonumber \\
		&\leq \frac{1}{n} \sum_{i \in \Nin} \norm{\nabla f_{i, I}(\vw)}^2 + \frac{1}{n} \sum_{i \in \Nout} \norm{\nabla f_{i, I}(\vw)}^2 + 3\norm{\vh_i(\vo_i, \vw)}^2, \label{eq:a0} \\
		&= \frac{1}{n}\sum_{i = 1}^{n} \norm{\nabla f_{i, I}(\vw)}^2 + 3\FracOut \frac{1}{\Nout}\sum_{i \in \Nout} \norm{\vh_{i}(\vo_i, \vw)}^2 \label{eq:a1},
	\end{align}
	where, in order to obtain~\eqref{eq:a0}, we use the Cauchy-Schwarz inequality and Assumption~\ref{as:low-signal-to-outlier}. 
	Since $f_{i, I}$ is $L$-smooth we have $\norm{\nabla f_{i, I}(\vw)}^2 \leq 2L (f_{i, I}(\vw) - f_{i, I}^{\ast})$. Using this fact and re-arranging terms in~\eqref{eq:a1} by adding and subtracting $f^{\ast}$ we obtain~\eqref{eq:lem:gradient-squares:eq01}.
		
	Use Lemma~\ref{lem:gradient} and expand $\E_{i}\left[\nabla f_i(\vw)\tran\nabla f_I(\vw) \right]$ as follows:
	\begin{align}
		\E_{i}\left[\nabla f_i(\vw)\tran\nabla f_I(\vw) \right] 
		&= \norm{\nabla f_I(\vw)}^2 + \FracOut \frac{1}{\Nout} \sum_{i \in \Nout} \vh_i(\vo_i, \vw)\tran \nabla f_I(\vw), \\
		&\geq 2\mu (f_I(\vw) - f^{\ast}_{I})  - \FracOut \frac{1}{\Nout} \sum_{i \in \Nout} \norm{\vh_i(\vo_i, \vw)}\norm{\nabla f_I(\vw)},  	
	\end{align}
	where (i) $\norm{\nabla f_I(\vw)}^2 \geq 2\mu (f_I(\vw) - f^{\ast}_{I})$ follows because $f_I$ is $\mu$-PL (as each of the $f_{i, I}$ are also $\mu$-PL), and (ii) the second part follows from  Cauchy-Schwarz inequality. This implies
	\begin{equation}
		\E_{i}\left[\nabla f_i(\vw)\tran\nabla f_I(\vw) \right] \geq~ 2\mu (f_I(\vw) - f^{\ast}_{I}). \nonumber 
	\end{equation}		
\end{proof}

\begin{lemma}
\label{lem:sigma-gradient-squares}
Let the low signal-to-outlier ratio assumption (Assumption~\ref{as:low-signal-to-outlier}) hold. Also, assume that $\frac{1}{n}\sum_{i=1}^{n} \sigma'_c(f_i(\vw)) = \zeta$.
Furthermore, let $\nabla f_{i, I}(\vw)\tran \nabla f_I(\vw) \geq 0$ and $0 < \phi \leq \inf_{w \in \setW} \sigma'(f_i(\vw))$ then
\begin{equation}
		\E_{i}\left[\sigma'_c(f_i(\vw))^2 \norm{\nabla f_i(\vw)}^2 \right] \leq 2 L (f_I(\vw) - f^{\ast}) + 2 L f_I^{\ast} \zeta  + 3 \FracOut \frac{1}{\Nout} \sum_{i \in \Nout} \sigma'_c(f_i(\vw))^2 \norm{\vh_i(\vo_i, \vw)}^2, \nonumber
\end{equation}	
and
\begin{equation}
	\E_i\left[ \sigma'_c(f_i(\vw)) \nabla f_i(\vw)\tran \nabla f_{I}(\vw)\right] \geq ~~2 \phi \mu (f_I(\vw) - f_{I}^{\ast}) \nonumber
\end{equation}
\end{lemma}
\begin{proof}
	The proof follows the same line of argument as the proof of Lemma~\ref{lem:gradient-squares}.
\end{proof}

\begin{lemma}
\label{lem:linear-convergence}
	Consider the recurrence relation $\delta_{t+1} \leq (1 - a \eta) \delta_t + \eta f(\eta) c$, where $a$, $\eta$, and $c$ are positive constants, and $f(\eta)$ is some known scalar function of $\eta$. We have $\delta_T < \epsilon$ provided
	\begin{equation}
		T > \frac{1}{a\eta} \log(2 \delta_0/\epsilon), ~~~f(\eta)c / a < \epsilon/2,~~~\text{and}~~~a\eta <1,
	\end{equation}
	for all $\epsilon > 0$.
\end{lemma}
\begin{proof}
	Let $a\eta <1$. Then, composing the iterates (\ie $\delta_{t+1} \leq (1 - a \eta) \delta_t + \eta f(\eta) c$) from $t = 0$ to $t=T$ we obtain
	\begin{align}
		\delta_T 
		&\leq (1 - a\eta)^T\delta_0 + \eta f(\eta) c \sum_{t=0}^{T-1}(1 - a\eta)^t, \\
		&\leq (1 - a\eta)^T \delta_0 + \eta f(\eta) c \sum_{t=0}^{\infty}(1 - a\eta)^t, \\
		&\leq (1 - a\eta)^T \delta_0 + \eta f(\eta) c \cdot \frac{1}{c \eta} =  (1 - a\eta)^T \delta_0 + \frac{c}{a}f(\eta).
	\end{align}
	Therefore, $\delta_T < \epsilon$ if $(1 - a\eta)^T < \epsilon/2$ and $f(\eta)c/a < \epsilon/2$. Once can deduce that $(1 - a\eta)^T \delta_0 < \epsilon/2$ if $\log(2\delta_0/\epsilon) < T\log(1 - a\eta)$, which happens if $\log(2\delta_0/\epsilon) < a\eta T$.
\end{proof}

\section{Proof of Theorem~\ref{thm:convergence-sgd}}
\label{pf:thm:convergence-sgd}
The functions $f_{i, I}$ are $L$-smooth. This implies that $f_I$ is also $L$-smooth. We can then write
\begin{equation}
	f_I(\vw_{t+1}) \leq f_{I}(\vw_t) + \nabla f_{I}(\vw_t)\tran(\vw_{t+1} - \vw_t) + \frac{L}{2}\norm{\vw_{t+1} - \vw_{t}}^2,
\end{equation}
for any $\vw_{t}, \vw_{t+1} \in \Real{d}$. Substituting them with the stochastic gradient descent updates, \ie $\vw_{t+1} = \vw_t - \eta \nabla f_i(\vw_t)$, where $i \sim \calU([n])$, we obtain
\begin{equation}
	f_I(\vw_{t+1}) \leq f_{I}(\vw_t) - \eta \nabla f_{I}(\vw_t)\tran \nabla f_i(\vw_t)+ \frac{L \eta^2}{2}\norm{ \nabla f_i(\vw_t) }^2. 
\end{equation}
Taking conditional expectation with respect to $\vw_t$, we obtain
\begin{equation}
	\E[f_I(\vw_{t+1})~|\vw_t] \leq f_{I}(\vw_t) - \eta ~\E_i[\nabla f_{I}(\vw_t)\tran \nabla f_i(\vw_t)]+ \frac{L \eta^2}{2}\E_i[\norm{ \nabla f_i(\vw_t) }^2]. 
\end{equation}
Using Lemma~\ref{lem:gradient-squares} we get:
\begin{multline}
	\E[f_I(\vw_{t+1})~|\vw_t] \leq f_{I}(\vw_t) - 2\eta \mu \left(1 - \frac{\eta L^2}{2\mu}\right)(f_I(\vw_t) - f^{\ast}_I) + \eta^2 L^2 \Delta_{f_I} \\
	+ \frac{3}{2}L \eta^2 \FracOut \frac{1}{\Nout}\sum_{i \in \Nout} \norm{\vh_i(\vo_i, \vw_t)}^2. \label{eq:v0}
\end{multline} 
We assume that $\vw_{t}, \vw_{t+1} \in \setW_{\text{SGD}}$. This implies $\frac{1}{\Nout}\sum_{i \in \Nout} \norm{\vh_i(\vo_i, \vw_t)}^2 \leq M$.
Taking expected value on both sides of~\eqref{eq:v0} %
and substituting $\delta_t = \E[f_I(\vw_t) - f^{\ast}_I]$ we obtain  
\begin{equation}
	\delta_{t+1} \leq (1 - \eta\mu)\delta_t +  \eta^2 \left(L^2 \Delta_{f_I} + \frac{3}{2}\FracOut L M\right) \label{eq:v2}
\end{equation}
where we have used $\eta < \mu / L^2$ to deduce $2\eta \mu \left(1 - \frac{\eta L^2}{2\mu}\right) > \eta\mu$. Using Lemma~\ref{lem:linear-convergence} establishes the result provided
\begin{equation}
	\eta < \min\left\{ \frac{1}{\mu}, \frac{\mu}{L^2}, \frac{\epsilon \mu}{3 \FracOut L M + 2 L^2 \Delta_{f_I} } \right\}.
\end{equation}
Noting that $\mu/L \leq 1$, this condition reduces to $\eta < \frac{\mu}{L}\min\left\{ \frac{1}{L}, \frac{\epsilon}{3\FracOut M + 2 L \Delta_{f_I}}\right\}$.

\section{Proof of Theorem~\ref{thm:convergence-adaptive}}
\label{pf:thm:convergence-adaptive}
The functions $f_{i, I}$ are $L$-smooth. This implies that $f_I$ is also $L$-smooth. We can then write
\begin{equation}
	f_I(\vw_{t+1}) \leq f_{I}(\vw_t) + \nabla f_{I}(\vw_t)\tran(\vw_{t+1} - \vw_t) + \frac{L}{2}\norm{\vw_{t+1} - \vw_{t}}^2,
\end{equation}
for any $\vw_{t}, \vw_{t+1} \in \Real{d}$. Substituting them with the adaptive algorithm's updates, \ie $\vw_{t+1} = \vw_t - \eta \sigma_{c_t}'(f_i(\vw_t))\nabla f_i(\vw_t)$, where $i \sim \calU([n])$, we obtain 
\begin{equation}
	f_I(\vw_{t+1}) \leq f_{I}(\vw_t) - \eta ~\sigma_{c_t}'(f_i(\vw_t))\nabla f_{I}(\vw_t) \tran \nabla f_i(\vw_t)+ \frac{L \eta^2}{2}\sigma_{c_t}'(f_i(\vw_t))^2\norm{ \nabla f_i(\vw_t) }^2. 
\end{equation}
Taking conditional expectation with respect to $\vw_t$, we obtain
\begin{multline}
	\E[f_I(\vw_{t+1})~|\vw_t] \leq f_{I}(\vw_t) - \eta \E_i[\sigma_{c_t}'(f_i(\vw_t)) \nabla f_{I}(\vw_t)\tran \nabla f_i(\vw_t)] \\
	+ \frac{L \eta^2}{2}\E_i[\sigma_{c_t}'(f_i(\vw_t))^2\norm{ \nabla f_i(\vw_t) }^2]. ~~~~~~~~~~
\end{multline}
Using Lemma~\ref{lem:sigma-gradient-squares} we get:  
\begin{multline}
	\E[f_I(\vw_{t+1})~|\vw_t] \leq f_{I}(\vw_t) - 2\eta \mu \beta \left(1 - \frac{\eta L^2}{2\mu \beta}\right)(f_I(\vw_t) - f^{\ast}_I) + \eta^2 L^2 \Delta_{f_I} \zeta \\
	+ \frac{3}{2}L \eta^2 \FracOut \frac{1}{\Nout}\sum_{i \in \Nout} \sigma_{c_t}'(f_i(\vw_t))^2 \norm{\vh_i(\vo_i, \vw_t)}^2. \label{eq:w0}
\end{multline} 
We assume that $\vw_{t}, \vw_{t+1} \in \setW_{\text{SGD}}$. This implies $\frac{1}{\Nout}\sum_{i \in \Nout} \sigma_{c_t}'(f_i(\vw_t))^2 \norm{\vh_i(\vo_i, \vw_t)}^2 \leq M$. 
Taking expected value on both sides of~\eqref{eq:w0} %
and substituting $\delta_t = \E[f_I(\vw_t) - f^{\ast}_I]$ we obtain
\begin{equation}
	\delta_{t+1} \leq (1 - \eta\mu\beta)\delta_t +  \eta^2 \left( L^2 \Delta_{f_I} \zeta + \frac{3}{2}L\eta^2\FracOut M\right) \label{eq:w2}
\end{equation}
where we have used $\eta < \mu \beta / L^2$ to deduce $2\eta \mu \beta \left(1 - \frac{\eta L^2}{2\mu \beta}\right) > \eta\mu \beta$. Using Lemma~\ref{lem:linear-convergence} establishes the result provided
\begin{equation}
	\eta < \min\left\{ \frac{1}{\mu\beta}, \frac{\mu\beta}{L^2}, \frac{\epsilon \mu\beta}{3 \FracOut L M + 2 L^2 \Delta_{f_I} \zeta } \right\}.
\end{equation}
Noting that $\mu/L \leq 1$, this condition reduces to $\eta < \frac{\mu \beta}{L}\min\left\{ \frac{1}{L}, \frac{\epsilon}{3\FracOut M + 2 L \Delta_{f_I} \zeta}\right\}$.

\section{Proof of Theorem~\ref{thm:convergence-gnc}}
\label{pf:thm:convergence-gnc}
Following the same line of arguments as in the proof of Theorem~\ref{thm:convergence-adaptive} (Appendix~\ref{pf:thm:convergence-adaptive}) we get
\begin{multline}
	\E[f_I(\vw_{t+1})~|\vw_t] \leq f_{I}(\vw_t) - 2\eta \mu \beta \left(1 - \frac{\eta L^2}{2\mu \beta}\right)(f_I(\vw_t) - f^{\ast}_I) + \eta^2 L^2 \Delta_{f_I} \zeta \\
	+ \frac{3}{2}L \eta^2 \FracOut \frac{1}{\Nout}\sum_{i \in \Nout} \sigma_{c_s}'(f_i(\vw_s))^2 \norm{\vh_i(\vo_i, \vw_t)}^2. \label{eq:w0}
\end{multline} 
for $t \in [s, s+T-1]$ and an $s \in \{0, T, 2T, \ldots \}$. We assume that $\vw_t, \vw_{t+1} \in \setW_{\sigma-\text{GNC}}$. This implies $ \FracOut \frac{1}{\Nout}\sum_{i \in \Nout} \sigma_{c_s}'(f_i(\vw_s))^2 \norm{\vh_i(\vo_i, \vw_t)}^2$ is bounded by $M$. Taking expected value on both sides of~\eqref{eq:w0} %
and substituting $\delta_t = \E[f_I(\vw_t) - f^{\ast}_I]$ we obtain
\begin{equation}
	\delta_{t+1} \leq (1 - \eta\mu\beta)\delta_t +  \eta^2 \left( L^2 \Delta_{f_I} \zeta + \frac{3}{2}L\eta^2\FracOut M\right). \label{eq:ww2}
\end{equation}
The result then follows from the same line of argument in Appendix~\ref{pf:thm:convergence-adaptive}.

\end{document}